\documentclass[letterpaper]{article} 
\usepackage{aaai23}  
\usepackage{times}  
\usepackage{helvet}  
\usepackage{courier}  
\usepackage[hyphens]{url}  
\usepackage{graphicx} 
\usepackage{natbib}  
\usepackage{caption} 
\usepackage[table,dvipsnames]{xcolor}
\usepackage{graphicx}

\makeatletter
\newcommand{\thickhline}{%
    \noalign {\ifnum 0=`}\fi \hrule height 1pt
    \futurelet \reserved@a \@xhline
}
\newcolumntype{"}{@{\hskip\tabcolsep\vrule width 1pt\hskip\tabcolsep}}
\makeatother

\usepackage{algorithm}
\usepackage{algpseudocode}
\algrenewcommand\algorithmicrequire{\textbf{Input:}}
\algrenewcommand\algorithmicensure{\textbf{Output:}}
\algnewcommand{\NULL}{\textsc{null}}

\usepackage{multirow}
\usepackage{amsmath}

\usepackage{booktabs}
\usepackage{amsthm}
\newtheorem{definition}{Definition}

\usepackage{amssymb}
\usepackage{makecell}
\usepackage[toc,page]{appendix}
\usepackage{xr}
\usepackage{subcaption}
\usepackage{arydshln}

\usepackage{newfloat}
\usepackage{listings}
\DeclareCaptionStyle{ruled}{labelfont=normalfont,labelsep=colon,strut=off} 
\floatstyle{ruled}
\newfloat{listing}{tb}{lst}{}
\floatname{listing}{Listing}

\title{Spectral Feature Augmentation for Graph Contrastive Learning and Beyond}
\author{
    Yifei Zhang$^{1}$, Hao Zhu$^{2,3}$, Zixing Song$^{1}$, Piotr Koniusz$^{3,2,}$\thanks{Corresponding author. PK was primarily concerned with the theoretical analysis (\eg, Prop.~\ref{prop:PowerIter1} \& \ref{prop:PowerIterVar}).}, Irwin King$^{1}$\\
}
\affiliations{
    \textsuperscript{\rm 1}The Chinese University of Hong Kong,
    \textsuperscript{\rm 2}Australian National University,
    \textsuperscript{\rm 3}Data61/CSIRO\\

    \texttt{\{yfzhang, zxsong, king\}@cse.cuhk.edu.hk; \\
    allenhaozhu@gmail.com; piotr.koniusz@data61.csiro.au}
}

\usepackage{bibentry}

\usepackage[bookmarks=false,colorlinks=true,urlcolor=black,citecolor={green!60!black},linkcolor={blue!80!black}]{hyperref}       

\usepackage[utf8]{inputenc} 
\usepackage[T1]{fontenc}    
\usepackage{url}            
\usepackage{booktabs}       
\usepackage{amsfonts}       
\usepackage{nicefrac}       
\usepackage{microtype}      
\usepackage{xcolor}         
\usepackage{wrapfig}
\usepackage{caption}
\usepackage{microtype}
\usepackage{graphicx}
\usepackage{booktabs} 
\usepackage{amsmath,bm}
\usepackage{multirow}
\usepackage{enumitem}
\usepackage{xspace}
\usepackage[most]{tcolorbox}
\usepackage{caption}
\usepackage{subcaption}

\usepackage{amsmath}
\usepackage{amssymb}
\usepackage{mathtools}
\usepackage{amsthm}
\usepackage{color, colortbl}
\usepackage{nameref}

\usepackage[capitalize,noabbrev]{cleveref}
\usepackage{multibib}
\newcites{latex}{References}

\usepackage{xcolor}
\usepackage{algcompatible}

\newtheorem{proposition}{Proposition}
\newtheorem{theorem}{Theorem}
\newtheorem{lemma}[theorem]{Lemma}

\usepackage[textsize=tiny]{todonotes}
\makeatletter
\DeclareRobustCommand\onedot{\futurelet\@let@token\bmv@onedotaux}
\def\bmv@onedotaux{\ifx\@let@token.\else.\null\fi\xspace}
\def\eg{\emph{e.g}\onedot} 
\def\ie{\emph{i.e}\onedot} 
\def\cf{\emph{c.f}\onedot} 
\def\etc{\emph{etc}\onedot} \def\vs{\emph{vs}\onedot}
\def\wrt{w.r.t\onedot}

\def\etal{\emph{et al}\onedot}

\makeatletter
\newcommand{\settitle}{\@maketitle}
\makeatother

\makeatletter
\def\ps@myheadings{%
    \let\@oddfoot\@empty\let\@evenfoot\@empty
    \def\@evenhead{\thepage\hfil\slshape\leftmark}%
    \def\@oddhead{{\slshape\rightmark}\hfil\thepage}%
    \let\@mkboth\@gobbletwo
    \let\sectionmark\@gobble
    \let\subsectionmark\@gobble
    }
  \if@titlepage
  \renewcommand\maketitle{\begin{titlepage}%
  \let\footnotesize\small
  \let\footnoterule\relax
  \let \footnote \thanks
  \null\vfil
  \vskip 60\p@
  \begin{center}%
    {\LARGE \@title \par}%
    \vskip 3em%
    {\large
     \lineskip .75em%
      \begin{tabular}[t]{c}%
        \@author
      \end{tabular}\par}%
      \vskip 1.5em%
    {\large \@date \par}
  \end{center}\par
  \@thanks\@notice
  \vfil\null
  \end{titlepage}%
  \setcounter{footnote}{0}%
}
\else
\renewcommand\maketitle{\par
  \begingroup
    \renewcommand\thefootnote{\@fnsymbol\c@footnote}%
    \def\@makefnmark{\rlap{\@textsuperscript{\normalfont\color{black}\@thefnmark}}}%
    \long\def\@makefntext##1{\parindent 1em\noindent
            \hb@xt@1.8em{%
                \hss\@textsuperscript{\normalfont\@thefnmark}}##1}%
    \if@twocolumn
      \ifnum \col@number=\@ne
        \@maketitle
      \else
        \twocolumn[\@maketitle]%
      \fi
    \else
      \newpage
      \global\@topnum\z@   
      \@maketitle
    \fi
    \thispagestyle{plain}
    \@thanks
    \if T\copyright@on\insert\footins{\noindent\footnotesize\copyright@text}\fi%
  \endgroup
  \setcounter{footnote}{0}%
}
\makeatother

\begin{document}
\maketitle
\begin{abstract}
%
%
Although  augmentations (\eg, perturbation of graph edges, image crops)  boost the efficiency of Contrastive Learning (CL), feature level augmentation is another plausible, complementary yet not well researched strategy. 
%
%
%
Thus, we present a novel spectral feature argumentation for contrastive learning on graphs (and images). To this end, for each data view, we estimate a low-rank approximation per feature map and  subtract that approximation from the map to obtain its complement. This is achieved by the proposed herein  incomplete power iteration, a non-standard power iteration regime which enjoys two valuable byproducts (under mere one or two iterations): (i) it partially balances spectrum of the feature map, and (ii) it injects the noise into rebalanced singular values of the feature map (spectral augmentation). 
For two views, we align these rebalanced feature maps as such an improved alignment step can focus more on less dominant singular values of matrices of both views, whereas the spectral augmentation does not affect the spectral angle alignment (singular vectors are not perturbed). 
We derive the analytical form for: (i) the incomplete power iteration to capture its spectrum-balancing effect, and (ii) the variance of singular values  augmented implicitly by the noise. We also show that the spectral augmentation improves the generalization bound. 
%
%
%
%
Experiments on graph/image datasets show that our spectral feature augmentation  outperforms 
baselines, and is complementary with other  augmentation strategies and compatible with various contrastive losses. 

\end{abstract}

\section{Introduction}
\begin{figure}[t]
\vspace{-0.7cm}
\centering
\includegraphics[width=0.35\textwidth]{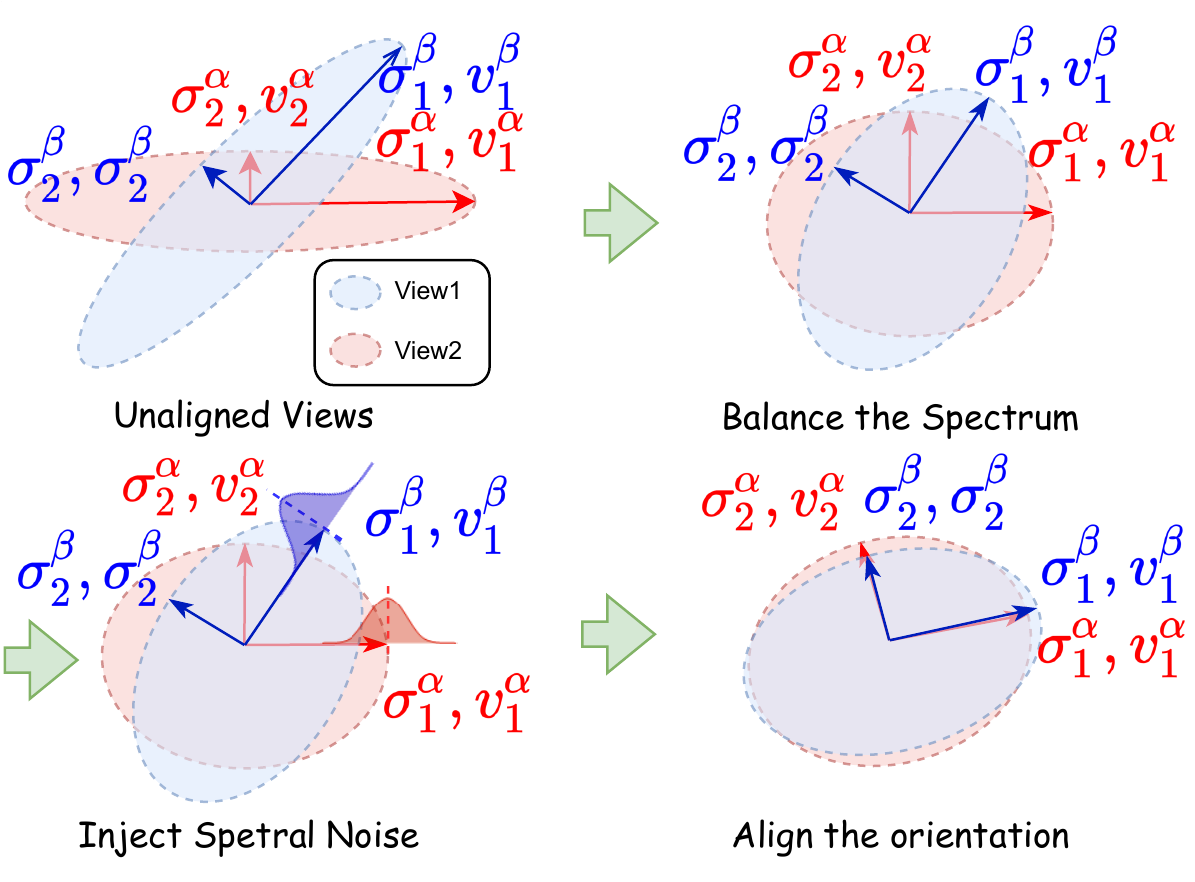}
\caption{Our GCL with spectral feature augmentation by the incomplete power iteration implicitly performs three steps. Let blue and red ellipses represent spectra of feature maps $\mathbf{H}^\alpha$ and $\mathbf{H}^\beta$  of two views. Singular values of unaligned views (top left) are firstly rebalanced (top right) and the noise is injected into singular values (bottom left). Given rebalanced singular values, corresponding singular vectors are aligned with equalized emphasis (leading unbalanced singular values would  emphasize the alignment of leading singular vectors, sacrificing the quality of alignment of other singular vectors).
\label{fig:implicit}
}
\vspace{-0.6cm}
\end{figure}
Semi-supervised and supervised Graph Neural Networks (GNNs) 
 \citep{velivckovic2017graph,hamilton2017inductive, song2022towards,zhang2022graph} require full access to class labels. However, unsupervised GNNs  \cite{klicpera2018predict,wu2019simplifying,zhu2021simple} and recent  Self-Supervised Learning (SSL) models do not  require  labels~\citep{song2021semi, DBLP:conf/ijcai/arvga} to train embeddings. Among SSL methods, Contrastive Learning (CL) achieves comparable performance with its supervised counterparts on many  tasks~\citep{chen2020simple,gao2021simcse}. CL has also been applied recently to the graph domain. A typical Graph Contrastive Learning (GCL) method forms multiple graph views via stochastic augmentation of the input to learn representations by contrasting so-called positive samples with negative samples~\citep{zhu2020deep,peng2020graph,zhu2021contrastive,zhu2021contrastive2,zhang2022costa}. As an indispensable part of GCL, the significance of graph augmentation has been well studied~\citep{hafidi2020graphcl,zhu2021graph,DBLP:journals/corr/autogcl}. 
%
%
Popular random data augmentations are just one strategy to construct views, and their noise may affect adversely downstream tasks~\citep{DBLP:journals/corr/AdversaGA,DBLP:conf/nips/whatGoodVeiw}. Thus, some works ~\citep{DBLP:journals/corr/autogcl, DBLP:conf/nips/whatGoodVeiw,DBLP:journals/corr/AdversaGA} learn graph augmentations but they require supervision.

The above issue motivates us to propose a simple/efficient data augmentation model 
which is complementary with existing augmentation strategies. 
We target Feature Augmentation (FA) as scarcely any FA works  exist in 
the context of CL and GCL.
In the image domain, a simple FA~\citep{DBLP:conf/cvpr/UpchurchGPPSBW17,DBLP:conf/icml/BengioMDR13} showed that perturbing feature representations of an image results in a representation of another image where both images share some semantics~\citep{wang2019implicit}. 
However, perturbing features randomly ignores covariance of feature representations, and ignores semantics correlations. 
Hence, we opt for injecting  random noise into the singular values of feature maps as such a spectral feature augmentation does not alter the orthogonal bases of feature maps by much, thus helping preserve semantics correlations. 

 \begin{figure*}
 \centering
 \begin{subfigure}[b]{0.65\textwidth}
   \includegraphics[width=0.9\textwidth]{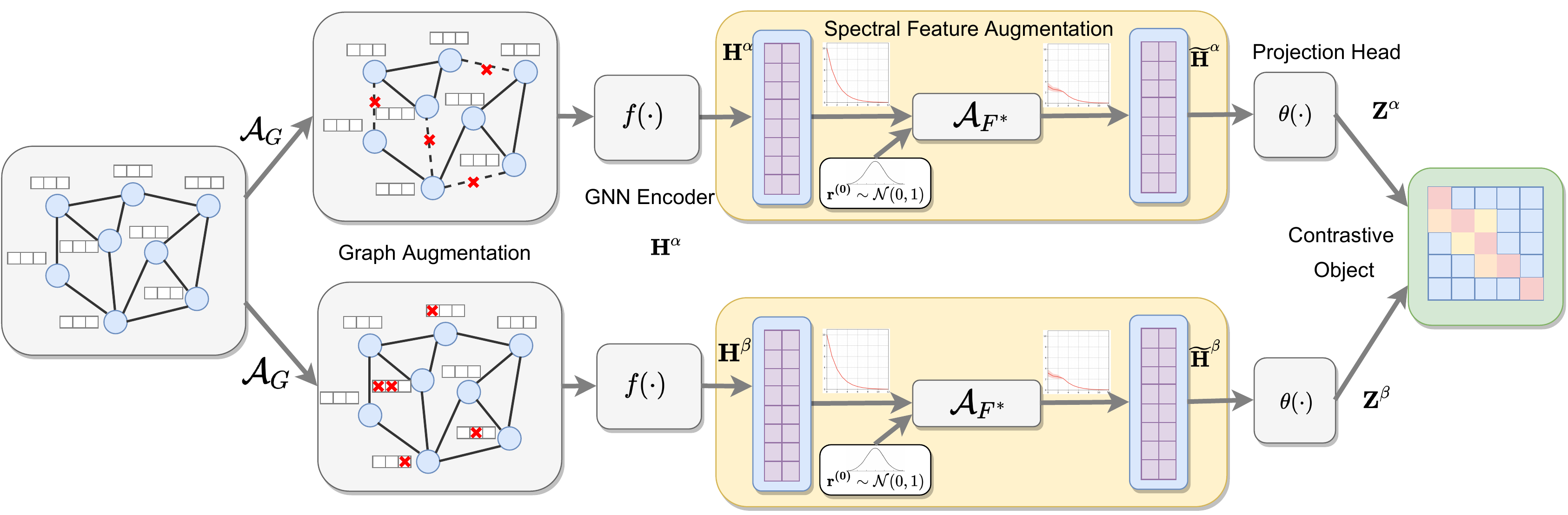}
    \caption{Our pipeline.\label{fig:framework1}}
    \end{subfigure}
    \begin{subfigure}[b]{0.25\textwidth}
    \includegraphics[width=\textwidth,height=3cm]{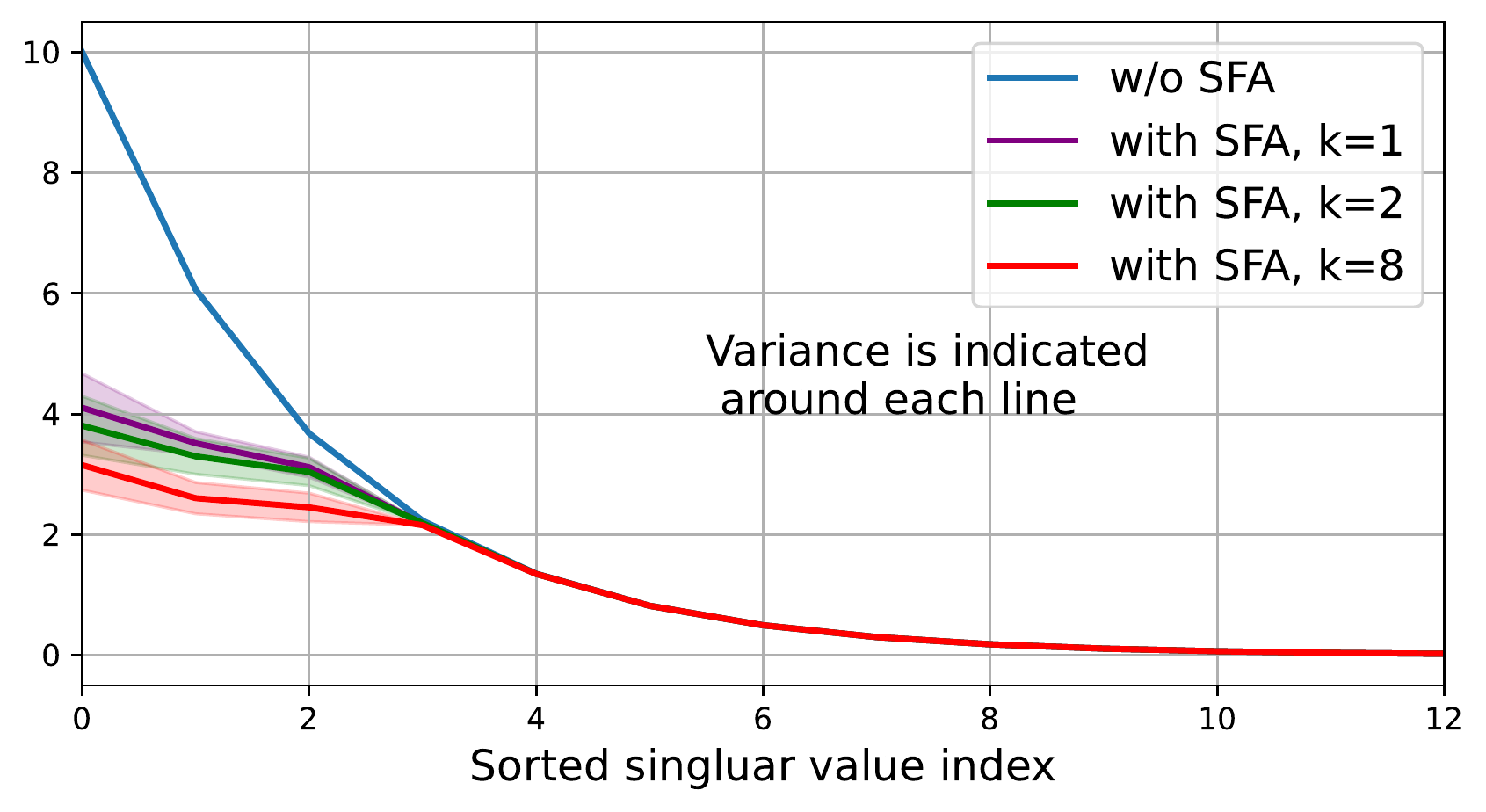}
    \caption{Simulation: spectrum obtained by Alg.~\ref{algo:1}.$\!\!\!$\label{fig:simulation}}
    \end{subfigure}
\vspace{-0.2cm}
\caption{Our GCL model. Two graph views are generated by data augmentation and passed into graph neural network encoders with shared parameters to learn node representations. The proposed spectral feature augmentation rebalances (partially equalizes) the spectrum of each feature map, and implicitly injects the noise into rebalanced singular values. Such representations are fed into the projection head and the contrastive loss. Figure~\ref{fig:implicit} explains the role of our spectral feature augmentation.}
\end{figure*}

Moreover, as typical GCL aligns two data views~\citep{wang2020understanding}, unbalanced singular values of two data views may affect the quality of alignment. As several leading singular values  (acting as weights on the loss) dominate the alignment process, GCL favors aligning the leading singular vectors of two data views while sacrificing remaining orthogonal directions with small singular values. In other words, the unbalanced spectrum leads to a suboptimal orthonormal bases alignment, which results in a suboptimal GCL model.

To address rebalancing of unbalanced spectrum and augmenting leading singular values, we present a novel and efficient {\em Spectral Feature Augmentation} (SFA).
To this end, we propose the so-called incomplete power iteration which, under just one or two iterations, partially balances singular values of feature maps and implicitly injects the noise into these singular values. 
We evaluate our method on various datasets for node level tasks (\ie, node classification and node clustering). We also
show that our  method is compatible with other augmentation strategies and contrastive losses.

\vspace{0.15cm}

We summarize our contributions as follows:

\renewcommand{\labelenumi}{\roman{enumi}.}
\begin{enumerate}[leftmargin=0.5cm]
     \item We propose a simple/efficient spectral feature augmentation for GCL which is independent of different contrastive losses, \ie, we employ 
     InfoNCE and Barlow Twin.
     \vspace{-0.1cm}
     \item We introduce the so-called incomplete power iteration which, under just one or two iterations, partially  balances  spectra of two data views and injects the augmentation noise into their singular values. The rebalanced spectra help align orthonormal bases  of both data views.
     \vspace{-0.1cm}
     \item As the  incomplete power iteration is stochastic in its nature, 
     we derive its analytical form which provably demonstrates its spectrum rebalancing effect in expectation, and captures the variance of the spectral augmentation.
     \vspace{-0.1cm}
     \item For completeness, we devise other spectral augmentation models, based on the so-called MaxExp and Power Norm. operators and Grassman feature maps, whose rebalancing and noise injection profiles differ with our method. 
 \end{enumerate}




\section{Related Work}
\label{sec:relatedwork}

\noindent\textbf{Data Augmentation.} 
Augmentations are usually performed in the input space. In computer vision, image transformations, \ie, rotation, flipping, color jitters, translation,  noise injection~\citep{shorten2019survey}, cutout and random erasure~\citep{devries2017dataset} are  popular. In neural language processing, 
token-level random augmentations, \eg, synonym replacement, word swapping, word insertion, and deletion~\citep{wei2019eda} are used. In transportation, conditional augmentation of road junctions is used \cite{coltrane}. 
In the graph domain, 
attribute masking, edge permutation, and node dropout are popular \citep{you2020graph}. Sun \etal \cite{uai_ke} use adversarial graph perturbations. Zhu \etal \citep{zhu2021graph} use adaptive graph augmentations based on 
the node/PageRank centrality~\citep{page1999pagerank} to mask  edges with varying probability.
%

\vspace{0.1cm}
\noindent\textbf{Feature Augmentation.} Samples can be augmented in the feature space instead of the input space~\citep{feng2021survey}. Wang \etal \citep{wang2019implicit} augment the hidden space features, resulting in auxiliary samples with the same class identity but different semantics. A so-called channel augmentation  perturbs the channels of feature maps~\citep{wang2019implicit} while  GCL approach, COSTA \cite{zhang2022costa}, augments features via random projections. Some few-shot learning approaches augment features  \cite{Zhang_2022_CVPR} while others  estimate the ``analogy'' transformations between samples of known classes to apply them on samples of novel classes \citep{hariharan2017low,schwartz2018delta} or mix foregrounds and backgrounds \cite{zhang2019few}. 
However, ``analogy'' augmentations are not applicable to contrastive learning due to the lack of labels.  

\vspace{0.1cm}
\noindent\textbf{Graph Contrastive Learning.} 
CL is popular in computer vision, NLP~\citep{he2020momentum, chen2020simple, gao2021simcse}, and graph learning. 
In the vision domain,  views are formed by augmentations at the pixel level, whereas in the graph domain, data augmentation may act on node attributes or the graph edges. 
GCL often explores node-node, node-graph, and graph-graph relations for contrastive loss which is similar to contrastive losses in computer vision. Inspired by SimCLR~\citep{chen2020simple}, GRACE~\citep{zhu2020deep}  correlates graph views by pushing closer representations of the same node in different views and separating representations of different nodes, and  Barlow Twin~\citep{zbontar2021barlow}   avoids the so-called dimensional collapse 
~\citep{jing2021understanding}. 

In contrast, we study spectral feature augmentations to perturb/rebalance singular values of both views. We outperform  feature augmentations such as COSTA~\cite{zhang2022costa}.

\section{Proposed Method}
Inspired by recent advances in 
augmentation-based GCL, our approach learns node representations by rebalancing spectrum of two data views and performing the spectral feature augmentation via the incomplete power iteration. SFA is complementary to the existing data augmentation approaches.  Figure~\ref{fig:framework1} illustrates our framework. The \nameref{sec:not} section (supplementary material) explains our notations.

\vspace{0.1cm}
\noindent\textbf{Graph Augmentation ($\mathcal{A}_G$).}  
Augmented graph $(\tilde{\mathbf{A}}, \tilde{\mathbf{X}})$ is generated by $\mathcal{A}_G$ by directly adding random perturbations to the original graph $({\mathbf{A}}, {\mathbf{X}})$.
Different augmented graphs are constructed given one input $(\mathbf {A}, \mathbf {X})$, yielding correlated views,  \ie, $(\widetilde{\mathbf{A}}^\alpha, \widetilde{\mathbf{X}}^\alpha)$ and $(\widetilde{\mathbf{A}}^\beta, \widetilde{\mathbf{X}}^\beta)$.
In the common GCL setting~\cite{zhu2020deep}, the graph structure is augmented by  permuting edges, whereas attributes by masking.

\vspace{0.1cm}
\noindent\textbf{Graph Neural Network Encoders.} 
Our framework admits various choices of the graph encoder. We opt for simplicity and adopt the commonly used graph convolution network (GCN)~\cite{kipf2016semi} as our base graph encoder. As shown in Fig.~\ref{fig:framework1}, we use a shared graph encoder for each view, \ie, $f: \mathbb{R}^{n \times d_{x}} \times \mathbb{R}^{n \times n}\longmapsto\mathbb{R}^{n \times d_{h}}$. We consider two graphs generated from $\mathcal{A}_G$  as two congruent structural views and define the GCN encoder with 2 layers as: 
 \begin{equation}
    \begin{aligned}
    &f(\mathbf{X}, \mathbf{A})=\mathrm{GCN}_{2}\left (\mathrm{GCN}_{1} (\mathbf{X}, \mathbf{A}), \mathbf{A}\right),\\
    &\quad\text{ where }\quad 
    \mathrm{GCN}_{l} (\mathbf{X}, \mathbf{A})=\sigma\big(\hat{\mathbf{D}}^{-\frac{1}{2}} \hat{\mathbf{A}} \hat{\mathbf{D}}^{-\frac{1}{2}} \mathbf{X} \Theta \big).
    \end{aligned}
\end{equation}
Moreover, $\tilde{\mathbf{A}}=\hat{\mathbf{D}}^{-1 / 2} \hat{\mathbf{A}} \hat{\mathbf{D}}^{-1 / 2} \in \mathbb{R}^{n \times n}$ is the degree-normalized adjacency matrix, $\hat{\mathbf{D}} \in \mathbb{R}^{n \times n}$ is the degree matrix of $\hat{\mathbf{A}}=\mathbf{A}+\mathbf{I}_{\mathbf{N}}$ where $\mathbf{I}_{\mathbf{N}}$ is the identity matrix, $\mathbf{X} \in \mathbb{R}^{n \times d_{x}}$ contains the initial node features, $\boldsymbol{\Theta} \in \mathbb{R}^{d_{x} \times d_{h}}$ contains network parameters, and $\sigma(\cdot)$ is a parametric ReLU (PReLU). 
The encoder outputs feature maps $\mathbf{H}^\alpha$ and $\mathbf{H}^\beta$ for two views.

\vspace{0.1cm}
\noindent\textbf{Spectral Feature Augmentation (SFA).}
$\mathbf{H}^\alpha$ and $\mathbf{H}^\beta$ are  fed to the feature augmenting function $\mathcal{A}_{F^*}$  where random noises are added to the spectrum via the incomplete power iteration. We explain the proposed SFA in the \nameref{sec:FA}~section and detail its properties in Propositions~\ref{prop:FeatureAug}, \ref{prop:PowerIter1} and \ref{prop:PowerIterVar}. SFA results in the spectrally-augmented feature maps, \ie, $\widetilde{\mathbf{H}}^\alpha$ and $\widetilde{\mathbf{H}}^\beta$. SFA is followed by a shared projection head $\theta: \mathbb{R}^{n \times d_{h}} \longmapsto \mathbb{R}^{n \times d_{z}}$ which is an MLP with two hidden layers and PReLU nonlinearity. It maps $\widetilde{\mathbf{H}}^\alpha$ and $\widetilde{\mathbf{H}}^\beta$ into  two node representations
$\mathbf{Z}^{\alpha}, \mathbf{Z}^{\beta} \in \mathbb{R}^{n \times d_{z}}$ (two congruent views of one graph) on which the contrastive loss is applied. As described in~\cite{chen2020simple}, it is beneficial to define the contrastive loss on $\mathbf{Z}$ rather than $\mathbf{H}$.

\vspace{0.1cm}
\noindent\textbf{Contrastive Training.}
To train the encoders end-to-end and learn rich node
 representations that are agnostic to downstream tasks, we utilize the InfoNCE loss~\cite{chen2020simple}: 
\begin{equation}
\label{equ:InfoNCE}
\begin{aligned}
&\mathcal{L}_{contrastive}(\tau)=\underbrace{\underset{\mathbf{z}, \mathbf{z}^{+}}{\mathbb{E}}\left[-\mathbf{z}^{\top} \mathbf{z}^{+}/\tau \right]}_{alignment}\\
&+ \underbrace{\underset{\mathbf{z}, \mathbf{z}^{+} 
}{\mathbb{E}}\Big[\log \Big(e^{{\mathbf{z}^\top\mathbf{z}^{+}}/{\tau}} \;\;+ \!\!\!\!\!\!\!\!\sum_{\mathbf{z}^{-} \in \mathbf{Z}^{\alpha\beta} \setminus \{\mathbf{z}, \
\mathbf{z}^{+}\}} \!\!\!\!\!\!\!\! e^{{\mathbf{z}^\top\mathbf{z}^{-}}/{\tau}}\Big)\Big]}_{uniformity},
\end{aligned}
\end{equation}
where $\mathbf{z}$ is the representation
of the anchor node in one view (\ie, $\mathbf{z} \in \mathbf{Z}^{\alpha}$) and $\mathbf{z}^{+}$ denotes the representation of the anchor node in another view (\ie,  $\mathbf{z}^{+} \in \mathbf{Z}^{\beta}$), whereas $\{\mathbf{z}^{-}\}$ are from the set of node representations other than $\mathbf{z}$ and $\mathbf{z}^{+}$ (\ie, $\mathbf{Z}^{\alpha\beta}\!\equiv\mathbf{Z}^\alpha \cup \mathbf{Z}^\beta$ and $\mathbf{z}^{-} \!\in \mathbf{Z}^{\alpha\beta}\setminus \{\mathbf{z}, \
\mathbf{z}^{+}\}$). The first part of  Eq.~\eqref{equ:InfoNCE} maximizes the alignment of two  views (representations of the same node become similar). The second part of  Eq.~\eqref{equ:InfoNCE} minimizes the pairwise similarity via LogSumExp. Pushing node representations away from each other  makes them uniformly distributed~\cite{wang2020understanding}. 
\begin{algorithm}[t]
   \caption{Spectral Feature Augmentation ($\mathcal{A}_{F^*}$)}
\begin{algorithmic}
   \STATE {\bfseries Input:} feature map $\mathbf{H}$; the number of iterations $k$;
   \STATE $\mathbf{r}^{(0)} \sim \mathcal{N}(0, \mathbf{I})$
   \FOR{$i=1$ {\bfseries to} $k$}
   \STATE $\mathbf{r}^{(i)} = \mathbf{H}^\top \mathbf{H} \mathbf{r}^{(i-1)}$
   \ENDFOR
   \STATE $\widetilde{\mathbf{H}} = \mathbf{H} -  \frac{\mathbf{H} \mathbf{r}^{(k)} {\mathbf{r}^{(k)}}^{\top}}{\|\mathbf{r}^{(k)}\|_2^2}$ 
\STATE \textbf{Return} $\widetilde{\mathbf{H}}$
\end{algorithmic}
\label{algo:1}
\end{algorithm}
\subsection{Spectral Feature Augmentation for GCL}
\label{sec:FA}
\begin{figure*}[t]
    \vspace{-0.5cm}
    \centering
    \begin{subfigure}[b]{0.22\textwidth}
    \includegraphics[width=\textwidth]{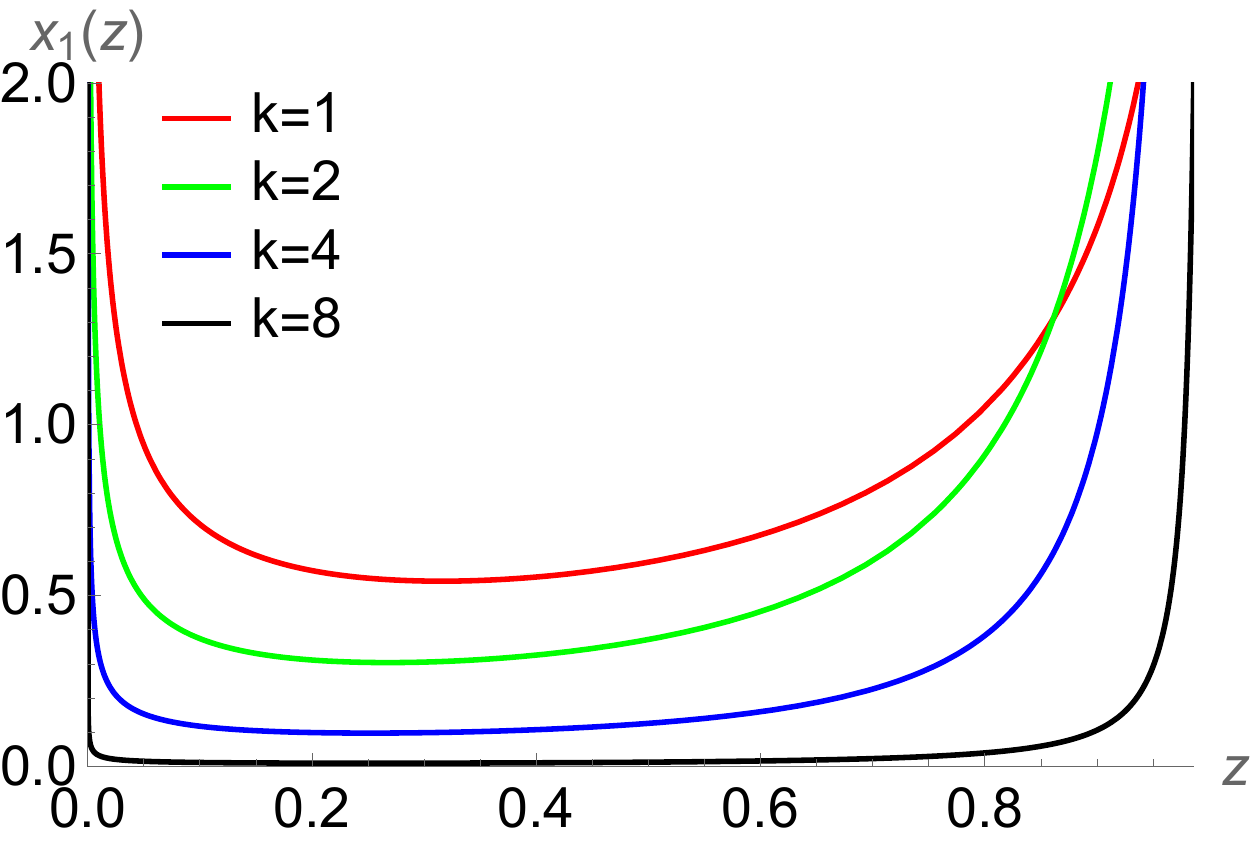}
    \caption{\label{fig:sim1}  }
    \end{subfigure}
    \hspace{10px}
    \begin{subfigure}[b]{0.22\textwidth}
    \includegraphics[width=\textwidth]{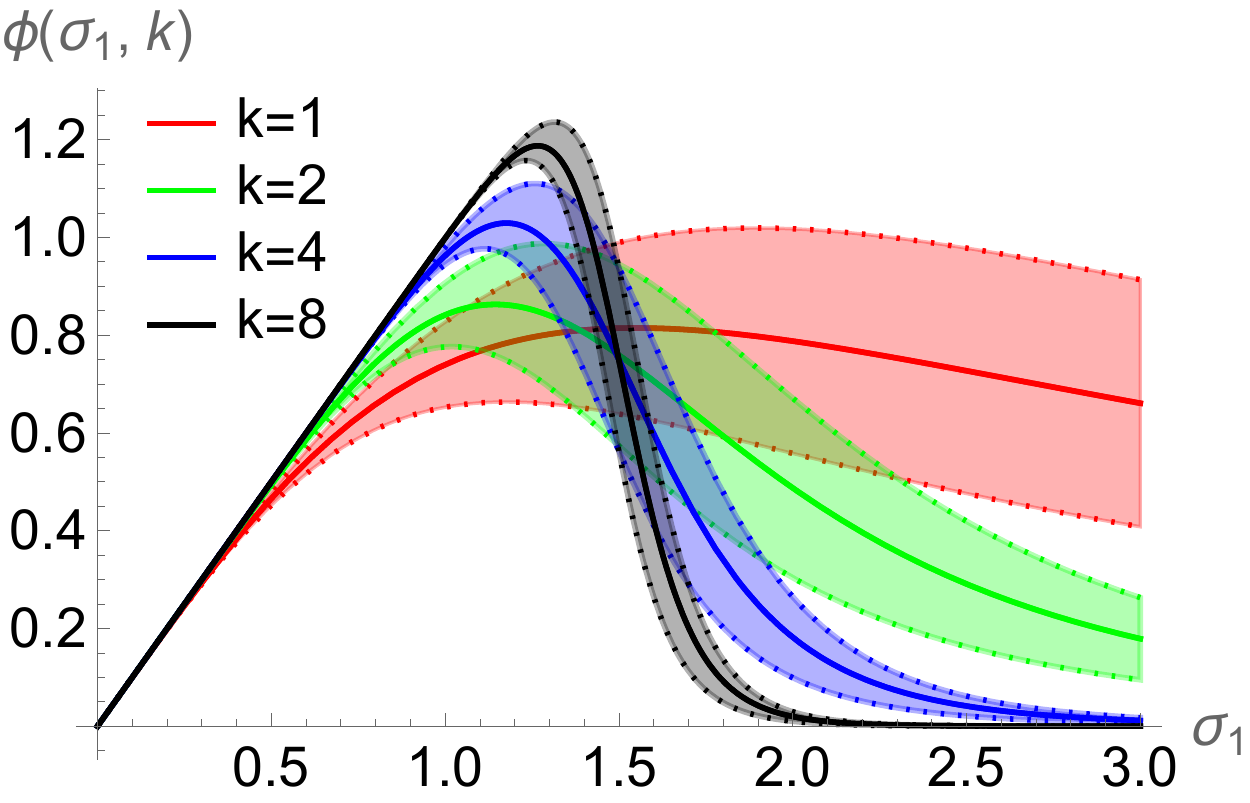}
    \caption{\label{fig:sim2}  }
    \end{subfigure}
     \hspace{10px}
    \begin{subfigure}[b]{0.22\textwidth}
    \includegraphics[width=\textwidth]{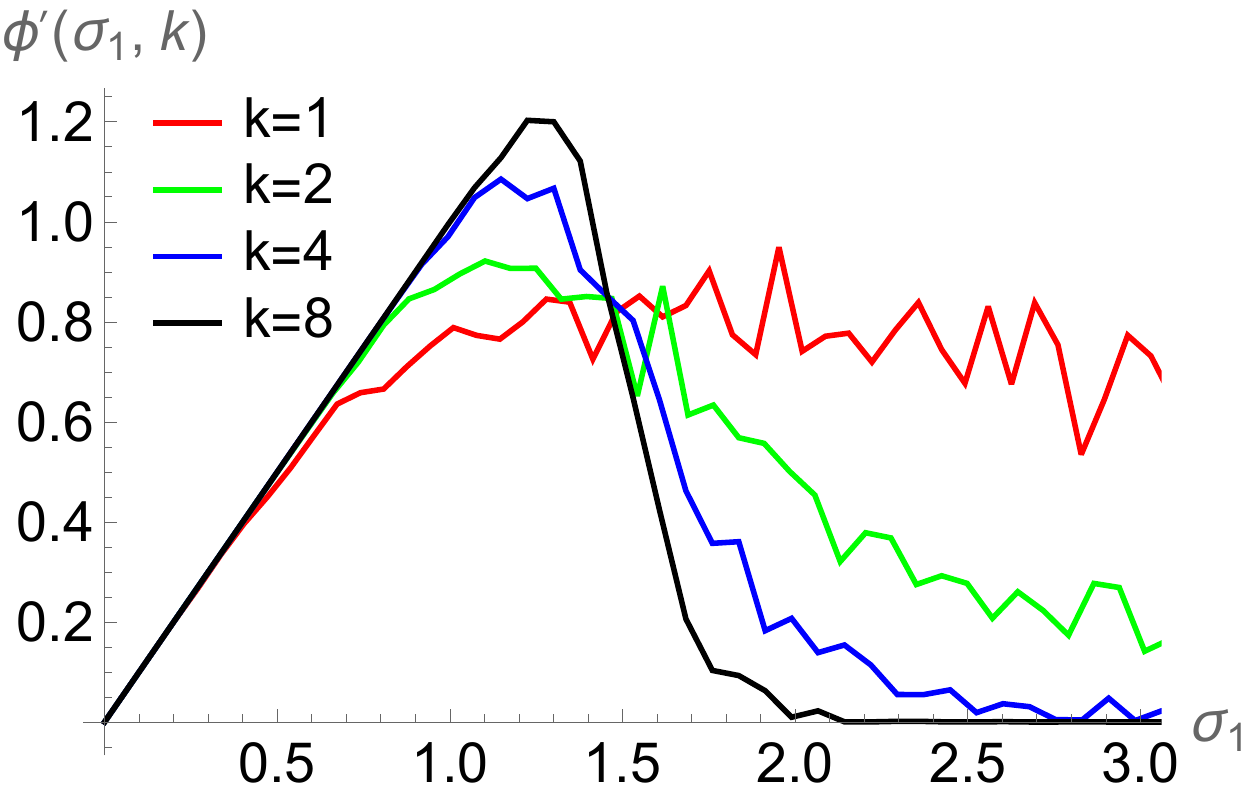}
    \caption{\label{fig:sim3}  }
\end{subfigure}
\vspace{-0.2cm}
\caption{Toy illustration of Prop. \ref{prop:PowerIter1} and \ref{prop:PowerIterVar}. Let $\sigma_2,\cdots,\sigma_5$ be $1.5,0.9,0.2,0.01$. We investigate the impact of  iterations $k\in\{1,2,4,8\}$. Fig. \ref{fig:sim1} shows distribution $x(z)$ given $\sigma_1=2$. Fig.  \ref{fig:sim2} shows the expected value $\phi(\sigma_1,k)=\sigma_1(1-\lambda_1)$ where $\lambda_1=\mathbb{E}_{z\sim x_1}(z)$ for $0\leq\sigma_1\leq3$. The deviation is indicated by $\phi_{\pm\omega_1}(\sigma_1,k)=\sigma_1(1-\lambda_1\pm\omega_1)$. Finally, Fig. \ref{fig:sim3}  is obtained via Alg. \ref{algo:1} (the incomplete power iteration). To this end, we generated randomly a feature matrix $\mathbf{H}$ 
and substituted its singular values by $\sigma_1,\cdots,\sigma_5$. Notice that for $k=1$, $1\leq\sigma_1\leq 3$, push-forward $\phi(\sigma_1,1)$ and $\phi'(\sigma_1,1)$ in Fig. \ref{fig:sim2} and \ref{fig:sim3} are around 0.8 (the balancing of spectrum) and the high deviation indicates the singular value undergoes  the spectral augmentation. For $k\geq 2$, both balancing and spectral augmentation effects decline. Note theoretical $\phi$ in Prop. \ref{prop:PowerIter1}  and real $\phi'$ from Alg. \ref{algo:1} match.
\label{fig:sim}
\vspace{-0.4cm}
}
\end{figure*}
Our spectral feature augmentation is inspired by the rank-1 update~\cite{yu2020toward}. 
Let  $\mathbf{H} = f(\mathbf{X}, \mathbf{A})$ be the graph feature map with the singular decomposition  $\mathbf{H} = \mathbf{U}\mathbf{\Sigma}\mathbf{V}^\top$ where $\mathbf{H} \in \mathbb{R}^{n \times d_h}$,  $\mathbf{U}$ and $\mathbf{V}$ are unitary matrices, and $\mathbf{\Sigma} = \text{diag}(\sigma_1, \sigma_2,\cdots, \sigma_{d_h})$ is the diagonal matrix with  singular values $\sigma_1 \geq \sigma_2 \geq \cdots \geq \sigma_{d_h}$.
Starting from a random point $\mathbf{r}^{(0)} \sim \mathcal{N}(0,\mathbf{I})$ and  function $\mathbf{r}^{(k)} = \mathbf{H}^\top \mathbf{H} \mathbf{r}^{(k-1)}$,
we generate a set of augmented feature maps\footnote{We apply Eq. \eqref{equ:FeatureAug} on both views $\mathbf{H}^\alpha$ and $\mathbf{H}^\beta$ separately to obtain spectrally rebalanced/augmented  $\widetilde{\mathbf{H}}^\alpha$ and $\widetilde{\mathbf{H}}^\beta$.} $\widetilde{\mathbf{H}}$ by:

\vspace{-0.4cm}
\begin{equation}
\label{equ:FeatureAug}
    \widetilde{\mathbf{H}}\big(\mathbf{H}; \mathbf{r}^{(0)}\big) =\mathbf{H} - \mathbf{H}_{\text{LowRank}} =  \mathbf{H} - \frac{\mathbf{H} \mathbf{r}^{(k)} \mathbf{r}^{(k)\top}}{\|\mathbf{r}^{(k)}\|_2^2}.
\end{equation}
\vspace{-0.2cm}

\noindent 
We often write $\widetilde{\mathbf{H}}$ rather than $\widetilde{\mathbf{H}}\big(\mathbf{H}; \mathbf{r}^{(0)}\big)$, and we often think of  $\widetilde{\mathbf{H}}$ as a matrix. We summarize the proposed SFA in Alg.~\ref{algo:1}.
\begin{proposition}
\label{prop:FeatureAug}
Let $\widetilde{\mathbf{H}}$ be the augmented feature matrix obtained via Alg.~\ref{algo:1} for the $k$-th iteration starting from a random vector $\mathbf{r}^{(0)}$ drawn from $\mathcal{N}(0,\mathbf{I})$. Then $\,\mathbb{E}_{\mathbf{r}^{(0)}\sim \mathcal{N}(0,\mathbf{I})} (\widetilde{\mathbf{H}}\big(\mathbf{H}; \mathbf{r}^{(0)}\big))=\mathbf{U} \widetilde{\Sigma} \mathbf{V}^{\top}\!$ has rebalanced spectrum\footnote{$\,$``Rebalanced'' means the output spectrum is flatter than the input.$\!\!\!\!$} $\;\widetilde{\boldsymbol{\Sigma}} = \text{diag}\big[(1-\lambda_1(k)\sigma_1, (1-\lambda_2(k))\sigma_2, \cdots, (1-\lambda_{d_h}(k))\sigma_{d_h}\big]\;$ where $\lambda_i(k)=\mathbb{E}_{\mathbf{y}\sim\mathcal{N}(0;\mathbf{I})}\big(\frac{\left(y_{i} \sigma_{i}^{2k}\right)^{2}}{ \sum_{l=1}^{d_h}\left(y_{l} \sigma_{l}^{2k}\right)^{2}})$ and $\mathbf{y} = \mathbf{V}^{\top}\mathbf{r}^{(0)}$, because $0\leq 1-\lambda_1(k) \leq 1-\lambda_2(k) \leq \cdots \leq 1-\lambda_{d_h}(k)\leq 1\;$ for $\;\sigma_1 \geq \sigma_2 \geq \cdots \geq \sigma_{d_h}$ (sorted singular values from the SVD), and so $(1-\lambda_i)$ gets smaller or larger  as $\sigma_i$ gets larger or smaller, respectively.
\end{proposition}



\definecolor{beaublue}{rgb}{0.88,15,1}
\definecolor{blackish}{rgb}{0.2, 0.2, 0.2}
\begin{tcolorbox}[width=1.0\linewidth, colframe=blackish, colback=beaublue, boxsep=0mm, arc=2mm, left=2mm, right=2mm, top=2mm, bottom=2mm]
\noindent\textbf{Push-forward Function (\textit{Partial balancing of spectrum}).} Prop. \ref{prop:FeatureAug} shows that in expectation, our incomplete power iteration rebalances spectrum according to the push-forward function  $\phi(\sigma_i;k)=\sigma_i(1\!-\!\lambda_i(k))$ where $\lambda_i(k)$ is an expected value of $\lambda'_i(\mathbf{y},k)=\frac{\left(y_{i} \sigma_{i}^{2k}\right)^{2}}{ \sum_{l=1}^{d_h}\left(y_{l} \sigma_{l}^{2k}\right)^{2}}$ \wrt  random variable $\mathbf{y}\!\sim\!\mathcal{N}(0, \mathbf{I})$ (see Eq. \eqref{eq:balanceq}). 
\end{tcolorbox}

\definecolor{beaublue}{rgb}{0.88,15,1}
\definecolor{blackish}{rgb}{0.2, 0.2, 0.2}
\begin{tcolorbox}[width=1.0\linewidth, colframe=blackish, colback=beaublue, boxsep=0mm, arc=2mm, left=2mm, right=2mm, top=2mm, bottom=2mm]
\textbf{Alg.~\ref{algo:1} returns an instance governed by the currently drawn $\mathbf{y}$}. The push-forward function in a feed-forward step of network realizes $\phi'(\sigma_i;\mathbf{y},k)=\sigma_i(1-\lambda'_i(\mathbf{y},k))$. Thus, below we study the analytical expression for $\phi(\sigma_i;k)$ and its variance to understand how drawing $\mathbf{y}\!\sim\!\mathcal{N}(0, \mathbf{I})$ translates into the variance posed by the implicit spectral augmentation of the singular values.
\end{tcolorbox}


\vspace{-0.1cm}
\begin{proposition}{{\em Analytical Expectation.}}
\label{prop:PowerIter1}
Let $\beta_i=(\sigma_i^{2k})^2$, then the expected value $\mathbb{E}_{\mathbf{y}\sim\mathcal{N}(0;\mathbf{I})}\frac{\beta_i y_{i}^2}{\beta_i y_{i}^2 +\sum_{l\neq i}\beta_l y^2_{l} } = \lambda_i(k)$ can be expressed as  $\mathbb{E}(x_i)$ over random variable $x_i\!=\!\frac{u}{u+v_i}$ for $u\!\sim\!\mathcal{G}(\frac{1}{2},2)$ and 
$v_i\!\sim\!\mathcal{G}(\alpha_i, 2\gamma_i)$ ($\mathcal{G}$ is Gamma distr.) with 
$\alpha_i\!=\!\frac{1}{2}\frac{(\sum_{l\neq i}\beta_l)^2 }{\sum_{l\neq i}\beta_l^2}$ and $\gamma_i\!=\!\frac{1}{\beta_i}\frac{\sum_{l\neq i}\beta_l^2 }{\sum_{l\neq i}\beta_l}$.
As PDF $x^\bullet_i(z)\!=\!\frac{\gamma_i}{(1-(1-z)\gamma_i)^2}\cdot\mathcal{B}\big(\frac{\gamma_i z}{1-(1-z)\gamma_i}; \frac{1}{2},\alpha_i\big)$ where  $\mathcal{B}$ is the Beta distribution and $x^\bullet_i(z)$ enjoys the support $z\in[0 ;1]$, then 
$\lambda_i=\mathbb{E}(z)=\int_0^1 z\cdot x^\bullet_i(z) \, \mathrm{d}z=\gamma_i^\frac{1}{2}\frac{\Gamma(\frac{3}{2})\Gamma(\frac{1}{2}+\alpha_i)}{\Gamma(\frac{1}{2})\Gamma(\frac{3}{2}+\alpha_i)}\cdot{_2F_1}\big(\frac{3}{2},\frac{1}{2}\!+\!\alpha_i,\frac{3}{2}\!+\!\alpha_i,1\!-\!\gamma_i\big)$ 
%
%
where ${_2F_1}(\cdot)$ is the so-called Hypergeometric function.
\end{proposition}
\begin{proof}
See \nameref{exp_val} (supplementary material).
\end{proof}

\begin{proposition}{{\em Analytical Variance.}}
\label{prop:PowerIterVar}
Following assumptions of Proposition \ref{prop:PowerIter1}, the variance $\omega^2_i$ of $x^\bullet_i(z)$ can be expressed as $\omega^2_i=\mathbb{E}(z^2)-(\mathbb{E}(z))^2=\int_0^1 z^2\cdot x^\bullet_i(z) \, \mathrm{d}z - \lambda^2_i=$ $0.56419\,\gamma^{\frac{1}{2}}\frac{\Gamma(\frac{1}{2}+\alpha_i)}{\Gamma(\alpha_i)}\allowbreak\big(0.4\cdot{_2F_1}\big(\frac{5}{2},1\!-\!\alpha_i,\frac{7}{2},1\big)\cdot \allowbreak {_2F_1}\big(\frac{5}{2},\frac{3}{2}\!+\!\alpha_i,\frac{5}{2}\!+\!\alpha_i,1\!-\!\gamma_i\big) +0.28571(\gamma_i-1)\cdot{_2F_1}\big(\frac{7}{2},1\!-\!\alpha_i,\frac{9}{2},1\big)\cdot{_2F_1}\big(\frac{7}{2},\frac{3}{2}\!+\!\alpha_i,\frac{7}{2}\!+\!\alpha_i,1\!-\!\gamma_i\big)\big)-\lambda^2_i$.
\end{proposition}
\begin{proof}
See \nameref{exp_var} (supplementary material).
\end{proof}

\definecolor{beaublue}{rgb}{0.88,1,1}
\definecolor{blackish}{rgb}{0.2, 0.2, 0.2}
\begin{tcolorbox}[width=1.0\linewidth, colframe=blackish, colback=beaublue, boxsep=0mm, arc=2mm, left=2mm, right=2mm, top=2mm, bottom=2mm]
\textbf{Note on Spectrum Rebalancing.}
Fig. \ref{fig:sim} explains the consequences of  Prop. \ref{prop:PowerIter1} \& \ref{prop:PowerIterVar} and connects them with Alg. \ref{algo:1}. Notice following: (i) For $k=1$ (iterations), the analytical form (Fig.~\ref{fig:sim2}) and the simulated incomplete power iteration (Fig. \ref{fig:sim3}) both indeed enjoy flatten $\phi$ and $\phi'$ for $1\leq\sigma_1\leq 3$. (ii) The injected variance is clearly visible in that flattened range (we know the quantity of injected noise). (iii) The analytical and simulated variances match.
\end{tcolorbox}

\definecolor{beaublue}{rgb}{0.88,1,1}
\definecolor{blackish}{rgb}{0.2, 0.2, 0.2}
\begin{tcolorbox}[width=1.0\linewidth, colframe=blackish, colback=beaublue, boxsep=0mm, arc=2mm, left=2mm, right=2mm, top=2mm, bottom=2mm]
\textbf{Choice of Number of Iterations ($k$).}
In Fig.~\ref{fig:sim2}, we plot $\phi(\sigma_i;k)$ using our analytical formulation. The best rebalancing effect is achieved for $k=1$ . For example, the red line ($k=1$) is mostly flat for $\sigma_i$ . This indicates the singular values $\sigma_i \geq 1$ are mapped to a similar value which promotes flattening of spectrum. When $\sigma_i \geq 2$ the green line eventually reduces to zero. This indicates that only datasets with spectrum falling into range between 1 and 1.2 will benefit from flattening. Important is to notice also that spectrum augmentation (variance) in Fig.~\ref{fig:sim2} is highest for $k=1$. Thus, in all experiments (including image classification), we set $k=1$, and the SFA becomes:
\vspace{-0.2cm}
\begin{equation}
\label{equ:FA_k=1}
\!\!\begin{aligned}
         \widetilde{\mathbf{H}} \!=\! \mathbf{H}\! -\!\mathbf{H}_{\text{LowRank}} \!=\!\mathbf{H}\bigg(\mathbf{I}\!-\!\frac{\mathbf{H}^{\top} \mathbf{H} \mathbf{r}^{(0)} \mathbf{r}^{(0) \top} \mathbf{H}^{\top} \mathbf{H}}{\|\mathbf{H}^{\top} \mathbf{H r}^{(0)}\|_2^2}\bigg).
    \end{aligned}\!\!\!\!\!
\end{equation}
\end{tcolorbox}


\vspace{-0.3cm}
\subsection{$\!$Why does the Incomplete Power Iteration work?$\!\!\!\!\!$}

Having discussed SFA, below we show how SFA improves the alignment/generalization  by flattening large and boosting small singular values due to rebalanced spectrum. 

\vspace{0.1cm}
\noindent 
\textbf{Improved Alignment.}
\label{sec:improveAligment}
SFA rebalances the weight penalty (by rebalancing singular values)  on orthonormal bases, thus improving the alignment of two correlated views. 
Consider the alignment part of Eq. \eqref{equ:InfoNCE} and ignore the projection head $\theta$ for brevity. The contrastive loss (temperature $\tau>0$,  $n$ nodes) on  $\mathbf{H}^\alpha$ and $\mathbf{H}^\beta$  
maximizes the alignment of two views:
\vspace{-0.1cm}
\begin{equation}
\label{equ:EignAlignOri}
\!
\begin{aligned}
\mathcal{L}_{a}& = \underset{\mathbf{h}^\alpha,\mathbf{h}^\beta}{\mathbb{E}}(\mathbf{h}^{\alpha\top}\mathbf{h}^\beta/\tau) = \frac{1}{n\tau}\text{Tr}(\mathbf{H}^{\alpha\top}\mathbf{H}^\beta)\\[-5pt]
&=\frac{1}{n\tau}\sum_{i=1}^{d_h} ({\sigma}^\alpha_i{\mathbf{v}}_{i}^{\alpha\top}\!{\mathbf{v}}^\beta_{i})({\sigma}^{\beta}_i {{\mathbf{u}}}^{\alpha\top}_{i}\!{\mathbf{u}}^{\beta}_{i} ).
\end{aligned}
%
\vspace{-0.2cm}
\end{equation}

\noindent 
The above equation indicates that for $\sigma^\alpha_i\geq0$ and $\sigma^\beta_i\geq0$, the maximum is reached if  the right and left singular value matrices are perfectly aligned, \ie, $\mathbf{U}^\alpha\! = \!\mathbf{U}^\beta$ and $\mathbf{V}^\alpha\! =\! \mathbf{V}^\beta$. Notice the singular values $\sigma^\alpha_i$ and $\sigma^\alpha_i$ serve as weighs for the alignment of singular vectors. 
As the singular value gap $\Delta\sigma_{12}=\sigma_1-\sigma_2$ is usually significant (spectrum of feature maps usually adheres to the power law $ai^{-\kappa}$ ($i$ is the index of sorted singular values, $a$ and $\kappa$ control the magnitude/shape), the large singular values tend to dominate the optimization. Such an issue makes  Eq. \eqref{equ:EignAlignOri} focus only on aligning the direction of dominant singular vectors, while neglecting remaining singular vectors, leading to a poor alignment of the orthonormal bases. In contrast,  SFA alleviates this issue. According to Prop.~\ref{prop:FeatureAug}, Eq. \eqref{equ:EignAlignOri} with SFA becomes:
\vspace{-0.1cm}
\begin{equation}
\label{equ:EignAlignRefine}
\!\!\!\begin{aligned}
\mathcal{L}^*_{a}& = \underset{\mathbf{h^\alpha},\mathbf{h^\beta}/\tau}{\mathbb{E}}\underset{\mathbf{r}^\alpha,\mathbf{r^\beta\sim \mathcal{N}(0,\mathbf{I})}}{\mathbb{E}}(\tilde{\mathbf{h}}^{\alpha\top}\tilde{\mathbf{h}}^\beta) \\[-4pt]
&=\frac{1}{n\tau}\sum_{i=1}^{d_h} (1\!-\!\lambda_i^\alpha){\sigma}^\alpha_i({\mathbf{v}}_{i}^{\alpha\top}{\mathbf{v}}^\beta_{i})(1\!-\!\lambda_i^\beta){\sigma}^{\beta}_i({{\mathbf{u}}}^{\alpha\top}_{i}{\mathbf{u}}^{\beta}_{i}).
\end{aligned}\!
\end{equation}

\vspace{-0.3cm}
\begin{tcolorbox}[width=1.0\linewidth, colframe=blackish, colback=beaublue, boxsep=0mm, arc=2mm, left=2mm, right=2mm, top=2mm, bottom=2mm]
Eq.~\eqref{equ:EignAlignRefine} shows that SFA limits the impact of leading singular vectors on the alignment step if SFA can rebalance the spectra. Indeed, Figure~\ref{fig:sim2} shows the spectrum balancing effect and one can see that $\phi(\sigma_i;k)=\sigma_i(1-\lambda_i(k))\leq\sigma_i$. The same may be concluded from $0\leq\lambda_i\leq1$ in Prop.~\ref{prop:PowerIter1}. See the~\nameref{lambda_upper}~section (supplementary material) for the estimated upper bound of $\phi$. Finally, the \nameref{sec:spectrum}~section also shows empirically that SFA leads to a superior alignment.
\end{tcolorbox}

\vspace{0.3cm}
\noindent\textbf{Improved Alignment Yields Better Generalization Bound.$\!\!\!$} 
To show SFA achieves the improved generalization bound we quote the following theorem \citep{DBLP:journals/corr/abs-2111-00743}.


\begin{theorem}
\label{theorem:gb}
Given a Nearest Neighbour classifier $G_f$, the downstream error rate of $G_{f}$ is 
$\operatorname{Err}\left(G_{f}\right) \!\leq\!(1\!-\!\sigma)\!+\!R_{\varepsilon},$ 
where
$R_{\varepsilon}\!\!= \!\!P_{\mathbf{x}_{1}, \mathbf{x}_{2} \in \mathcal{A}(\mathbf{x})}\{\|f\left(\mathbf{x}_{1}\right)\!-\!f\left(\mathbf{x}_{2})\| \!\geq\! \varepsilon\right\}\!\leq\!\frac{\sqrt{2-2\mathcal{L}_{a}}}{\varepsilon}$, $\sigma$ is the parameter of the so-called
 $(\sigma, \delta)$-augmentation (for each latent class, the proportion of samples located in a ball with diameter $\delta$ is larger than $\sigma$, $\mathcal{A}(\cdot)$ is the set of augmented samples, $f(\cdot)$ is the encoder, and $\{\|\cdot\|\!\geq\!\varepsilon\}$ is the set of samples with $\varepsilon$-close representations among augmented data. 
\end{theorem}
\vspace{-0.2cm}
\begin{proof}
See \nameref{gen_bound_proof} (supp. material).
\end{proof}
\vspace{-0.1cm}

\noindent 
Theorem~\ref{theorem:gb} says the key to better generalization of contrastive learning is  better alignment $\|f(\mathbf{x}_{1})-f(\mathbf{x}_{2})\|$ of positive samples. SFA improves alignment by design. See \nameref{sec:improveAligment} ~~See empirical result in the \nameref{sec:spectrum}~section. Good alignment (\eg, Fig. \ref{fig:EigenAlign}) due to spectrum rebalancing (\eg, Fig. \ref{fig:Eigen1}) enjoys 
 $\mathcal{L}_{a}\leq\mathcal{L}^*_{a}$ (Eq. \eqref{equ:EignAlignOri} and \eqref{equ:EignAlignRefine}) so
 one gets $\!R_{\varepsilon}^{*}\!\leq \!R_{\varepsilon}$ and the lower generalization bound  $\operatorname{Err}\big(G^{*}_{f}\big) \leq \operatorname{Err}\big(G_{f}\big)$. Asterisk $^*$ means SFA is used ($\mathcal{L}^*_{a}$ replaces $\mathcal{L}_{a}$).

\definecolor{Gray}{gray}{0.9}
\definecolor{LightCyan}{rgb}{0.88,1,1}

\definecolor{DarkBlue}{rgb}{0,0, 0.4}

\section{Experiments}

\begin{table*}[t]
\vspace{-0.5cm}
\begin{minipage}[h]{0.67\textwidth}
\vspace{-0.5cm}
\centering
\resizebox{\textwidth}{!}{
\begin{tabular}{lcccccc}
\toprule
\textbf{Method}& \textbf{WikiCS} & \textbf{Am-Comput.} & \textbf{Am-Photo} & \textbf{Cora} & \textbf{CiteSeer} & \textbf{PubMed} \\
\midrule 
RAW fatures & $71.98 \pm 0.00$ & $73.81 \pm 0.00$ & $78.53 \pm 0.00$ & $64.61 \pm 0.22$ & $65.77 \pm 0.15$ & $82.02 \pm 0.26$\\
DeepWalk & $74.35 \pm 0.06$ & $85.68 \pm 0.06$ & $89.44 \pm 0.11$ & $74.61 \pm 0.22$ & $50.77 \pm 0.15$ & $80.11 \pm 0.25$\\
\midrule
DGI   & $75.35 \pm 0.14$ & $83.95 \pm 0.47$ & $91.61 \pm 0.22$  & $82.15 \pm 0.63$ & $69.51 \pm 0.52$  & $86.01 \pm 0.26$            \\
MVGRL & $77.52 \pm 0.08$ & $87.52 \pm 0.11$ & $91.74 \pm 0.07$  & $83.11 \pm 0.12$ & $73.33 \pm 0.03$  & $84.27 \pm 0.04$            \\           
GRACE & $78.19 \pm 0.48$ & $87.25 \pm 0.25$ & $92.15 \pm 0.25$  & $83.51 \pm 0.25$ & $73.63 \pm 0.20$ & $85.51 \pm 0.37$             \\
GCA   & $78.35 \pm 0.05$ & $87.85 \pm 0.31$ & $92.49 \pm 0.16$  & $82.89 \pm 0.21$ & $72.89 \pm 0.13$ & $85.12 \pm 0.23$             \\
SUGRL & $77.72 \pm 0.28$ & $88.83 \pm 0.23$  & $93.28 \pm 0.42$ & $83.47 \pm 0.55$ & $73.08 \pm 0.45$ & $84.91 \pm 0.31$\\
MERIT & $77.92 \pm 0.43$ & $87.53 \pm 0.26$  & $93.12 \pm 0.42$ &$84.11  \pm 0.65$ &$74.34  \pm 0.43$ & $84.12 \pm 0.23$\\
BGRL  & $79.11 \pm 0.62$ & $87.37 \pm 0.40$ & $91.57 \pm 0.44 $ & $83.77 \pm 0.57$ & $73.07 \pm 0.06$ & $84.62 \pm 0.35$\\ 
G-BT  & $76.85 \pm 0.62$ & $86.86 \pm 0.33$ & $92.63 \pm 0.57$  & $83.63 \pm 0.44$ & $72.95 \pm 0.17$ & $84.52 \pm 0.12$\\ 
COSTA  &$79.12  \pm 0.02$ & $88.32 \pm 0.03$ & $92.56 \pm 0.45$  &$84.32  \pm 0.22$ & $72.92 \pm 0.31$ &$86.01\pm0.22$\\
\midrule 
\rowcolor{LightCyan} SFA$_{\text{BT}}$ & $\textcolor{blue}{\mathbf{80.22} \pm 0.05}$ & $\textcolor{blue}{88.14 \pm 0.15}$ & $\textcolor{blue}{92.83 \pm 0.14}$ & $\textcolor{blue}{84.10 \pm 0.01}$ & $\textcolor{blue}{73.73 \pm 0.03}$& $\textcolor{blue}{85.63 \pm 0.07}$\\
\rowcolor{LightCyan} SFA$_{\text{InfoNCE}}$ & $\textcolor{blue}{79.98 \pm 0.05}$ & $\textcolor{blue}{\mathbf{89.24 \pm 0.27}}$ & $\textcolor{blue}{\mathbf{93.53 \pm 0.16}}$ & $\textcolor{blue}{\mathbf{85.89 \pm 0.01}}$ & $\textcolor{blue}{\mathbf{75.31 \pm 0.03}}$ & $\textcolor{blue}{\mathbf{86.29 \pm 0.13}}$\\
\bottomrule
\end{tabular}
}
\caption{\label{tab:MainResultNodeClassifcation}Node classification on graph datasets. Note that SFA$_{\text{InfoNEC}}$ and SFA$_{\text{BT}}$ can be directly compared to GRACE and G-BT. (Accuracy is reported.)}
\vskip 0.1in
\begin{minipage}[c]{0.59\textwidth}
\begin{minipage}[c]{0.49\textwidth}
\vspace{0.1cm}
    \centering
    \includegraphics[width=1\textwidth]{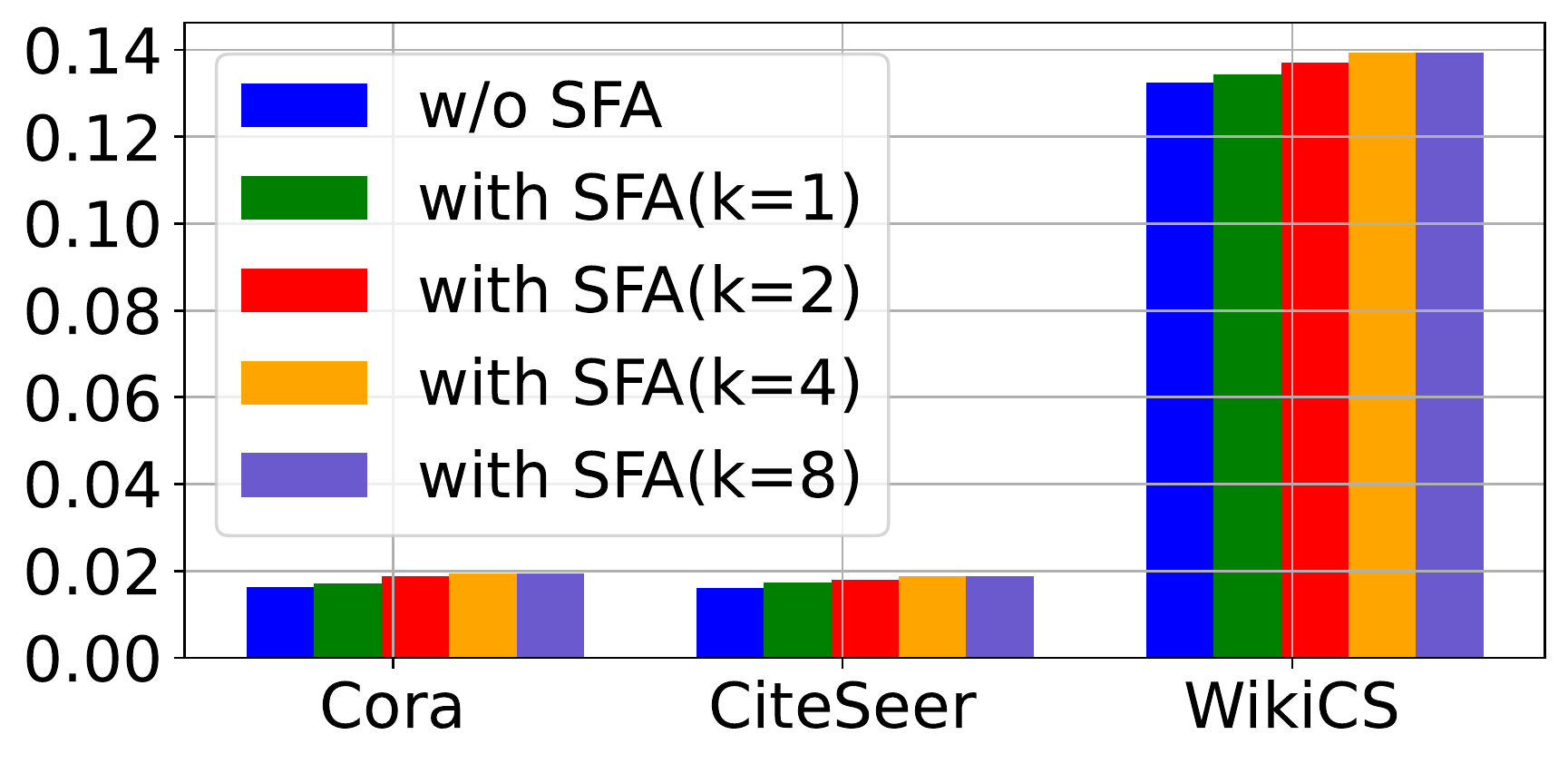}
    \vspace{-0.7cm}
    \captionof{figure}{Running time per epoch in seconds. \label{fig:runningtime}}
	\end{minipage}
\begin{minipage}[c]{0.49\textwidth}
\vspace{0.05cm}
\centering
\resizebox{\textwidth}{!}{
\begin{tabular}{l|c|c}
\toprule
\multirow{2}{*}{\textbf{Datasets}}  &    \multirow{2}{*}{{w/o SFA}}     & \multicolumn{1}{c}{{SFA}}  \\
& & $k= 1$  \\
\midrule
\textbf{Cora}       & $165$ &$172$  \\
\textbf{CiteSeer}   &$164$  &$173$  \\
\textbf{WikiCS}     &$1334$  &$1343$ \\
\bottomrule
\end{tabular}
}
\vspace{-0.12cm}
\captionof{table}{Running time per epoch in seconds.\label{tab:runningtime}}
\end{minipage}
\begin{minipage}[c]{\textwidth}
 \resizebox{\textwidth}{!}{
    \begin{tabular}{cccc}
    \toprule
    Method &    SFA (ours)       &  SVD      & Random SVD \\
    \midrule
    Time   &    0.25 hour        &  12 hours & 1.2 hours  \\
    \bottomrule
    \end{tabular}
    }
    \captionof{table}{\label{tab:time_comp_svd}Running time on Ogb-arxiv.}
%
\end{minipage}
\end{minipage}
\begin{minipage}[c]{0.403\textwidth}
\resizebox{1\textwidth}{!}{
\begin{tabular}{l|cc}
\toprule
  \textbf{Ogb-arxiv}      & \textbf{Validation}     &\textbf{Test}\\
  \midrule
  MLP                                   & $57.6 \pm 0.2$     & $55.5 \pm 0.3$ \\
  Node2vec                              & $71.2 \pm 0.3$     & $70.0 \pm 0.3$ \\
  MVGLR                                 & $69.3 \pm 0.3$     & $68.2 \pm 0.2$ \\
  DGI                                   & $71.2 \pm 0.1$     & $70.3 \pm 0.1$ \\
  SUGRL                                 & $70.2 \pm 0.1$     & $69.3 \pm 0.2$ \\
  MERIT                                 & $67.2 \pm 0.1$     & $65.3 \pm 0.2$\\
  GRACE                                 & $71.4 \pm 0.5$     & $70.8 \pm 0.1$\\
  G-BT                                  & $71.1 \pm 0.3$     & $70.0 \pm 0.2$ \\
  COSTA                                  & $71.6 \pm 0.4$     & $71.0 \pm 0.4$ \\
 \rowcolor{LightCyan} SFA$_{\text{Info}}$           & \textcolor{blue}{$\bf 72.3 \pm 0.1$} & \textcolor{blue}{$\bf 71.6 \pm 0.4$} \\
  \bottomrule
    \end{tabular}
    }
\caption{Node classification.\label{tab:obgarxiv}}
\end{minipage}
\end{minipage}
\hspace{0.2cm}
\begin{minipage}[h]{0.33\textwidth}
\vspace{-0.5cm}
\centering
        \resizebox{0.95\textwidth}{!}{
             \setlength{\tabcolsep}{3pt}
    \begin{tabular}{lccc}
    \toprule
         \textbf{Method}        & \textbf{CIFAR10} & \textbf{CIFAR100} & \textbf{ImageNet-100}\\
         \hline
          SimCLR      &  $90.5$   &  $65.5$  & $76.8$ \\
         \rowcolor{LightCyan} SFA$_{\text{SimCLR}}$    &  \textcolor{blue}{$\mathbf{91.6}$}   &  \textcolor{blue}{$\mathbf{66.7}$} & $\textcolor{blue}{\mathbf{77.7}}$\\ 
         \midrule
         BalowTw      &  $92.0$   &  $69.7$ &  $80.0$  \\
         \rowcolor{LightCyan} SFA$_{\text{BT}}$   &  \textcolor{blue}{$\mathbf{92.5}$}   &  \textcolor{blue}{$\mathbf{70.4}$}   & \textcolor{blue}{$\mathbf{80.9}$} \\
         \midrule
 Siamese               &      $90.51$     &    $66.04$    &      $74.5$       \\
 \rowcolor{LightCyan}SFA$_{\text{Siamese}}$          &      \textcolor{blue}{$\mathbf{91.23}$}    &    \textcolor{blue}{\textbf{66.99}}    &     \textcolor{blue}{$\mathbf{75.6}$}     \\
 \midrule
 SwAV                  &      $89.17$     &    $64.88$    &     $74.0$       \\
 \rowcolor{LightCyan} SFA$_{\text{SwAV}}$            &      \textcolor{blue}{$\mathbf{90.12}$}     &    \textcolor{blue}{$\mathbf{65.82}$}    &      \textcolor{blue}{$\mathbf{74.8}$}       \\
    \bottomrule
    \end{tabular}
    }
    \vspace{-0.15cm}
    \caption{Image classification on CIFAR10, CIFAR100 and ImageNet-100.}
    \label{tab:image_result}

   \vspace{0.2cm}
    \resizebox{0.95\textwidth}{!}{
    \begin{tabular}{lccc}
\toprule
         \textbf{Method}        & \textbf{NCI1}   & \textbf{PROTEIN}  & \textbf{DD} \\
         \midrule
         GraphCL                   &$77.8 \pm 0.4$ &$74.3 \pm 0.4$ &${77.8 \pm 0.4}$ \\
         LP-Info                   &$75.8  \pm 1.2$   &$\textcolor{blue}{74.6 \pm 0.2}$ &$72.5 \pm 1.9$ \\
         JOAO                      &$\textcolor{blue}{78.0 \pm 0.4}$ &${74.5 \pm 0.4}$ &$\textcolor{blue}{77.5 \pm 0.5}$ \\
         SimGRACE                  &$77.4 \pm 1.0$   &$73.9 \pm 0.1$  &$77.3 \pm 1.1$ \\
         \rowcolor{LightCyan} SFA$_{\text{InfoNEC}} $     &$\textcolor{blue}{\mathbf{78.7 \pm 0.4}}$   &$\textcolor{blue}{\mathbf{75.4 \pm 0.4}}$  &$\textcolor{blue}{\mathbf{78.6 \pm 0.4}}$ \\
    \bottomrule
    \end{tabular}
    }
    \vspace{-0.15cm}
    \caption{Graph classification. \color{DarkBlue}{(SFA$_{\text{InfoNEC}}$ can be directly compared with  GraphCL.)} \label{tab:graphcls}}
 \vspace{0.2cm}
\resizebox{0.95\textwidth}{!}{
\begin{tabular}{cccccc}
\toprule
$\mathcal{A}_{G}$ & $\mathcal{A}_{F}$ & $\mathcal{A}_{F^*}$ & \textbf{Am-Comput.} & \textbf{Cora} & \textbf{CiteSeer} \\
\midrule 
$\times$     &$\times$ &$\times$         &$85.01$ &$80.04$   &$71.35$\\
$\checkmark$ &$\times$ &$\times$         &$87.25$ &$82.23$   &$74.56$\\
$\times$     &$\times$ &$\checkmark$     &$86.74$ &$84.19$   &$73.15$\\
\rowcolor{LightCyan}$\checkmark$ &$\times$ &$\checkmark$     &$\textcolor{blue}{\mathbf{88.74}}$ &$\textcolor{blue}{\mathbf{85.90}}$   &$\textcolor{blue}{\mathbf{75.05}}$\\
\midrule 
$\checkmark$     &$\checkmark$ &$\times$         &$87.55$& $83.35$&  $74.45$\\
\rowcolor{LightCyan} $\checkmark$     &$\checkmark$ &$\checkmark$    &$\textcolor{blue}{{88.69}}$ & $\textcolor{blue}{{84.50}}$&  $\textcolor{blue}{{74.70}}$\\
\bottomrule
\end{tabular}
}
\vspace{-0.15cm}
\caption{\label{tab:AblationStudy}Ablation study on different augmentation strategies: $\mathcal{A}_G$, $\mathcal{A}_F$ and $\mathcal{A}_{F^*}$.}
\end{minipage}
\vspace{-0.3cm}
\end{table*}

Below we conduct  experiments on 
the node classification, node clustering, graph classification and image classification. 
For fair comparisons, we use the same experimental setup as the representative Graph SSL (GSSL) methods (\ie, GCA~\cite{zhu2020deep} and GRACE~\cite{zhu2021graph}). 

\vspace{0.1cm}
\noindent\textbf{Datasets.}
We use five popular 
datasets~\cite{zhu2020deep, zhu2021graph,velickovic2019deep}, including citation networks (Cora, CiteSeer) and social networks (Wiki-CS, Amazon-Computers, Amazon-Photo)~\cite{kipf2016semi,sinha2015overview,mcauley2015image, mernyei2020wiki}. 
For graph classification, we use NCI1, PROTEIN and DD \cite{dobson2003distinguishing,riesen2008iam}. For image classification we use CIFAR10/100~\cite{krizhevsky2009learning} and ImageNet-100~\cite{DBLP:conf/cvpr/DengDSLL009}.
See the \nameref{app:modelsetting}~section for details (supplementary material).

\vspace{0.1cm}
\noindent\textbf{Baselines.}
We  focus on three groups of SSL models. The first group includes traditional GSSL, \ie, Deepwalk~\cite{perozzi2014deepwalk}, node2vec~\cite{grover2016node2vec}, and  GAE~\cite{kipf2016variational}. The second group is contrastive-based GSSL, \ie,  
Deep Graph Infomax (DGI)~\cite{velickovic2019deep}, Multi-View Graph Representation Learning (MVGRL)~\cite{hassani2020contrastive}, GRACE~\cite{zhu2020deep}, GCA~\cite{zhu2021graph}, ~\cite{jin2021multi} SUGRL~\cite{mo2022simple}. The last group  
does not require explicit negative samples, \ie, Graph Barlow Twins(G-BT)~\cite{bielak2021graph} and BGRL~\cite{thakoor2021bootstrapped}. We also compare SFA with  COSTA~\cite{zhang2022costa} (GCL with feat. augmentation).

\vspace{0.1cm}
\noindent\textbf{Evaluation Protocol.} 
We adopt the  
evaluation from~\cite{velickovic2019deep, zhu2020deep, zhu2021graph}. Each model is trained in an unsupervised manner on the whole graph with node features. Then, we pass the raw features into the trained encoder to obtain   embeddings and train an $\ell_2$-regularized logistic regression classifier. Graph Datasets are randomly divided into 10\%, 10\%, 80\% for training, validation, and testing. We report the accuracy with mean/standard deviation over 20 random data splits.

\vspace{0.1cm}
\noindent \textbf{Implementation details} We use Xavier initialization for the GNN parameters and train the model with Adam optimizer. For node/graph classification, we use 2 GCN layers. The logistic regression classifier is trained with $5,000$ (guaranteed converge). We also use early stopping with a patience of 20 to avoid overfitting. We set the size of the hidden dimension of nodes to from ${128}$ to $512$. In clustering, we train a k-means clustering model. For the chosen hyper-parameters see Section~\ref{app:hyperparam}. We implement the major baselines using PyGCL~\cite{DBLP:journals/corr/pygcl}. The detailed settings of augmentation and contrastive objectives are in Table~\ref{tab:ModelAugSetting} of Section~\ref{app:modelsetting}. 



\subsection{Main Results}
\noindent\textbf{Node Classification\footnote{Implementation and evaluation are based on  PyGCL \cite{DBLP:journals/corr/pygcl}: \url{https://github.com/PyGCL/PyGCL}.}.}
We employ  node classification as a downstream task to showcase  SFA. The default contrastive objective is InfoNCE or the BT loss. Table~\ref{tab:MainResultNodeClassifcation} shows that SFA consistently achieves the best results on all datasets. Notice that the graph-augmented  GSSL methods, including GSSL with SFA, significantly outperform the traditional methods, illustrating the importance of data augmentation in GSSL. Moreover, we find that the performance of GCL methods (\ie, GRACE, GCA, DGI, MVGRL) improves by a large margin when integrating with SFA (\eg, major baseline, GRACE, yields 3\% and 4\% gain on  Cora and Citeseer), which shows that SFA  is complementary to graph augmentations. SFA  
 also works with the BT loss and  improves its performance.

\vspace{0.1cm}
\noindent\textbf{Graph Classification.}
By adopting the graph-level GNN encoder, SFA can be used for the graph-level pre-taining. 
Thus, we compare SFA with graph augmentation-based models (\ie, GarphCL \citep{DBLP:conf/nips/YouCSCWS20}, JOAO~\citep{DBLP:conf/icml/YouCSW21}) and augmentation-free models (\ie, SimGrace~\cite{DBLP:conf/www/XiaWCHL22}, LP-Info~\cite{DBLP:conf/wsdm/YouCWS22}). Table~\ref{tab:graphcls} shows that SFA outperforms all baselines on the three datasets.

\vspace{0.1cm}
\noindent\textbf{Image Classification\footnote{Implementation/evaluation are based on Solo-learn~\cite{JMLR:v23:21-1155}: \url{https://github.com/vturrisi/solo-learn}.$\!\!\!\!$}.}
As SFA perturbs the spectrum of feature maps, it is also applicable to 
the image domain. Table~\ref{tab:image_result}  presents the top-1 accuracy on CIFAR10/100 and ImageNet-100. See also the \nameref{sec:impl} (supp. material).

\vspace{0.1cm}
\noindent\textbf{Runtimes.}  
Table~\ref{tab:runningtime} and Fig.~\ref{fig:runningtime} show SFA incurs a negligible runtime overhead  (few matrix-matrix(vector) multiplications). 
The cost of SFA for $k\!<\!8$ iterations 
is negligible.




To rebalance the spectrum, choices for a push-forward function whose curve gets  flatter as singular values get larger  are many but they require SVD to rebalance singular values directly. 
The complexity of SVD is $\mathcal{O}(nd_h^2)$ for $n\!\gg\!d_h$, which is costly in GCL as  SVD has to be computed per mini-batch (and SVD is unstable in back-prop. due to non-simple singular values ).  Table~\ref{tab:time_comp_svd} shows the runntime on  Ogb-arixv. 

\subsection{Analysis on Spectral Feature Augmentation}
\label{sec:spectrum}
\noindent\textbf{Improved Alignment.}
We show that SFA improves the alignment of two views during training. Let $\mathbf{H}^\alpha$ and $\mathbf{H}^\beta$ ($\widetilde{\mathbf{H}}^\alpha$ and $\widetilde{\mathbf{H}}^\beta$) denote the feature maps of two views without (with) applying SFA. The alignment is computed by $\|\mathbf{H}^\alpha -\mathbf{H}^\beta\|^2_F$ (or $\|\tilde{\mathbf{H}}^\alpha -\tilde{\mathbf{H}}^\beta\|^2_F$). Fig.~\ref{fig:EigenAlign}  shows that SFA achieves better alignment. See also the \nameref{sec:moreresult}.

\vspace{0.1cm}
\noindent\textbf{Empirical Evolution of Spectrum.}
Fig.~\ref{fig:Eigen1} (Cora) shows how the singular values of  features maps ${\mathbf{H}}$ (without SFA) and $\tilde{\mathbf{H}}$ (with SFA) evolve. As training progresses, the gap between consecutive singular values gradually decreases due to SFA. The leading components of spectrum (where the signal is) become more balanced: empirical results match Prop.~\ref{prop:FeatureAug} \&  \ref{prop:PowerIter1}.

\vspace{0.1cm}
\noindent\textbf{Ablations on Augmentations.}
Below, 
we compare graph augmentation $\mathcal{A}_G$, channel feature augmentation $\mathcal{A}_F$ (random noise  added to embeddings directly) and the spectral feature augmentation $\mathcal{A}_{F^*}$.  Table~\ref{tab:AblationStudy} shows that using $\mathcal{A}_G$ or $\mathcal{A}_{F^*}$ alone  improves performance. For example, $\mathcal{A}_G$ yields 2.2\%, 2.2\% and 3.2\% gain on Am-Computer, Cora, and CiteSeer. $\mathcal{A}_{F^*}$ yields 1.7\%, 4.1\% and 1.8\% gain on  Am-Computer, Cora, and CiteSeer. Importantly, when both $\mathcal{A}_G$ and $\mathcal{A}_{F^*}$ are applied, the gain is \textbf{3.7}\%, \textbf{6.0}\% and \textbf{4.8}\% on Am-Computer, Cora, and CiteSeer over ``no augmentations''. Thus, SFA is complementary to existing graph augmentations. We also notice that $\mathcal{A}_{F^*}$ with $\mathcal{A}_G$ outperforms $\mathcal{A}_{F}$ with $\mathcal{A}_G$ by 1.1\%, 2.4\% and 1.7\%, which highlights the benefit of SFA.

\vspace{0.1cm}
\noindent\textbf{Effect of Number of Iterations ($k$).}
Below, we analyze how $k$ in Eq.~\eqref{equ:FeatureAug} influences the performance. We set $k \in \{0, 1, 2, 4, 8\}$ in our model with the InfoNCE loss on Cora, Citeseer and Am-Computer. The case of $k=0$ means that no power iteration is used, \ie, the solution simplifies to the feature augmentation by subtracting from $\mathbf{H}$ perturbation  $\mathbf{H}\mathbf{r}^{(0)}{\mathbf{r}^{(0)}}^\top\!\!/\|\mathbf{r}^{(0)}\|_2^2$ with random $\mathbf{r}^{(0)}$. 
Table~\ref{tab:EffectofK} shows that without power iteration, the performance of the model drops , \ie, 1.5\%, 3.3\% and 2.2\% on  Am-Comp., Cora and Citeseer. 

\vspace{-0.1cm}
\begin{tcolorbox}[width=1.0\linewidth, colframe=blackish, colback=beaublue, boxsep=0mm, arc=2mm, left=2mm, right=2mm, top=2mm, bottom=2mm]
The best gain in Table ~\ref{tab:EffectofK} is achieved for $1\leq k\leq 2$. This is consistent with Prop. \ref{prop:PowerIter1}, Fig. \ref{fig:sim2} and \ref{fig:sim3}, which show that the spectrum balancing effect (approximately flat region of $\phi$) and significant spectrum  augmentation (indicated by the large deviation in Fig. \ref{fig:sim2}) are achieved only when the power iteration is incomplete (low $k$, \ie, $1\leq k\leq 2$).
\end{tcolorbox}
\vspace{-0.1cm}

\begin{figure}
\begin{minipage}[c]{0.5\textwidth}
\begin{minipage}[b]{0.47\textwidth}
\includegraphics[width=\columnwidth]{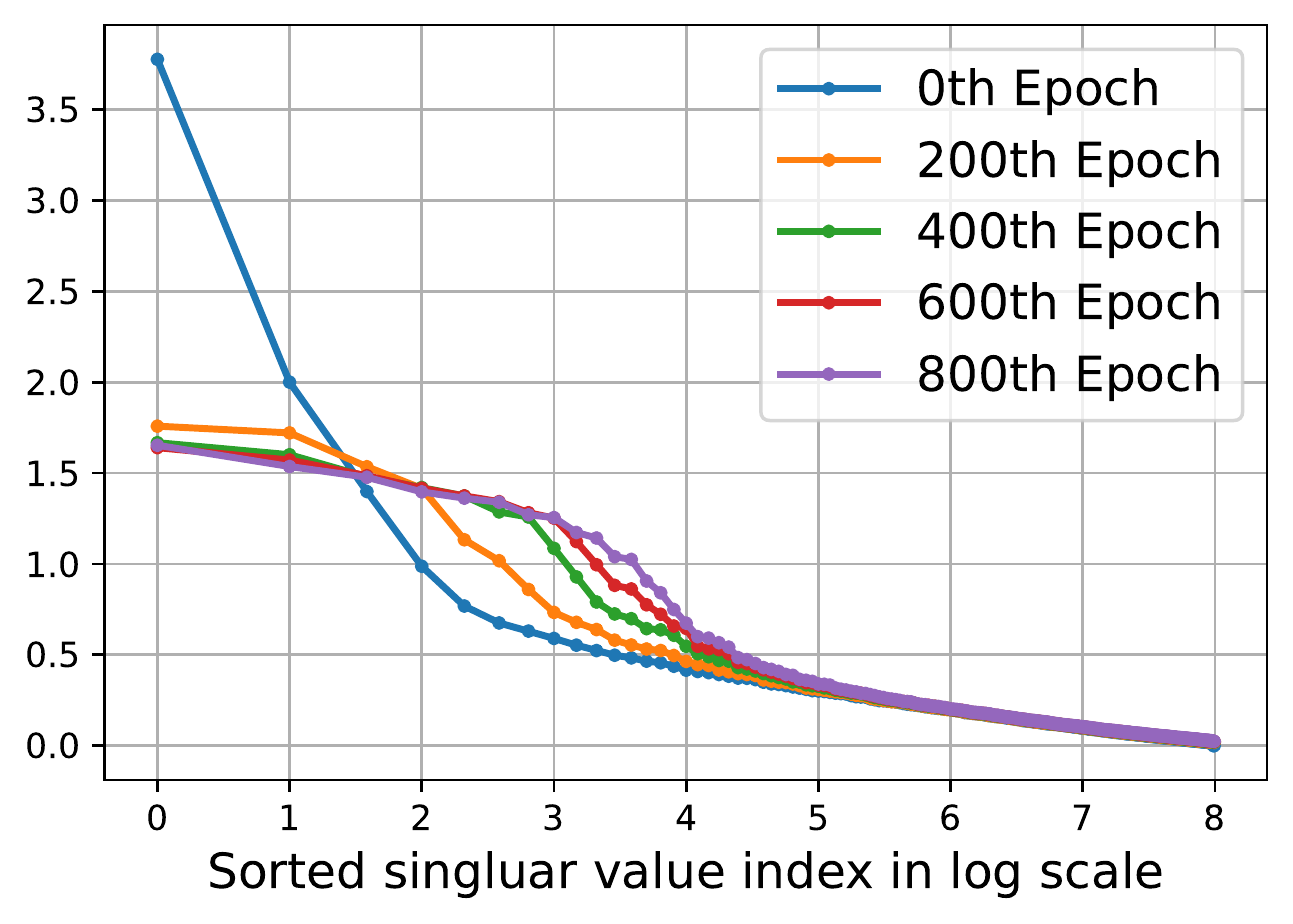}
\vspace{-0.7cm}
\caption{Spectrum of $\tilde{\mathbf{H}}$ when training on Cora with SFA.}
\vspace{0.2cm}
\label{fig:Eigen1}
\end{minipage}
\hspace{2px}
\begin{minipage}[b]{0.48\textwidth}
    \includegraphics[width=\columnwidth]{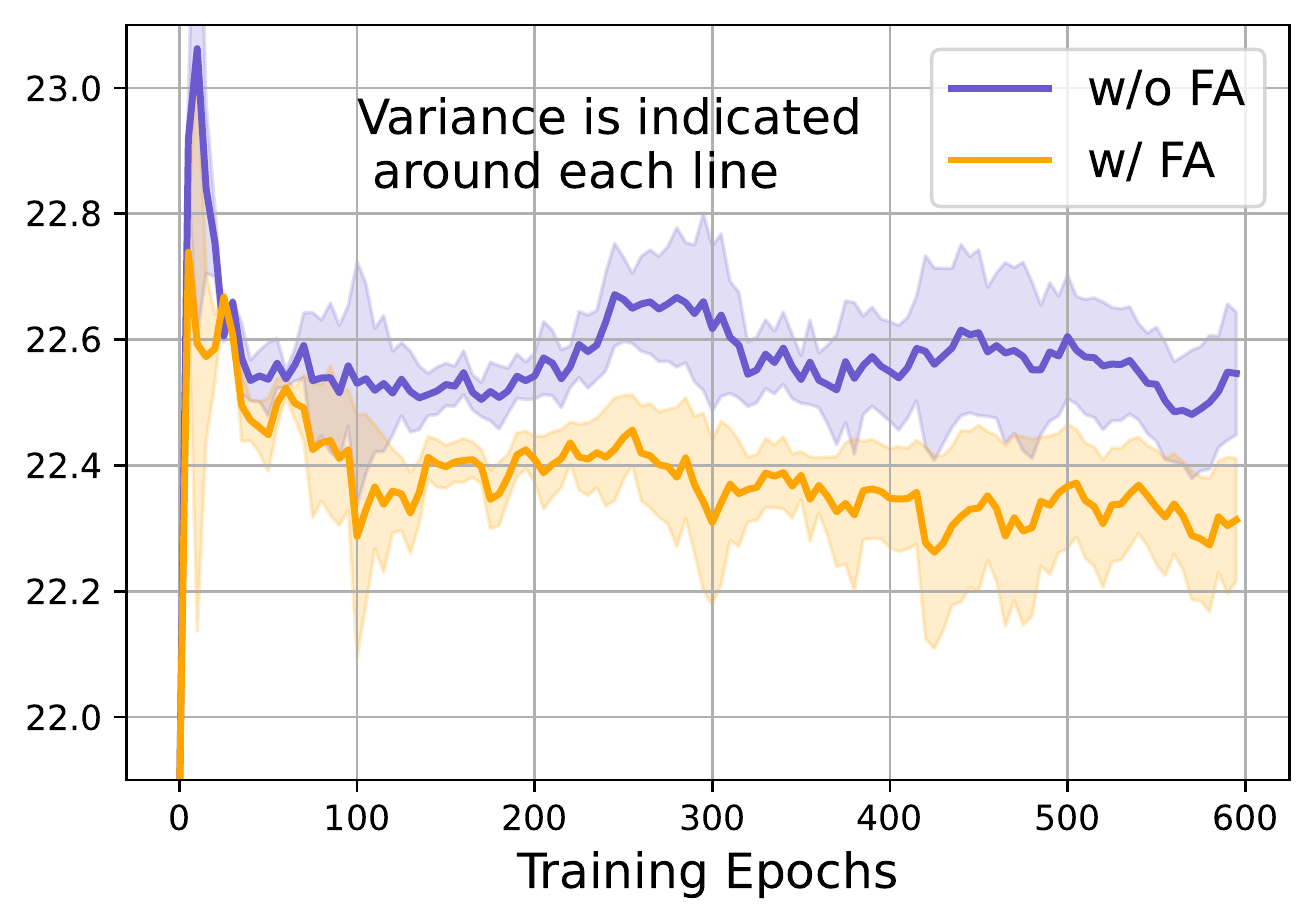}
    \vspace{-0.7cm}
    \caption{Alignment of  f. maps of two views (lower is better).}
    \label{fig:EigenAlign}
    \vspace{0.2cm}
\end{minipage}
\end{minipage}
\begin{minipage}[c]{0.5\textwidth}
\vspace{2px}
    \begin{subfigure}[b]{0.32\textwidth}
     \includegraphics[width=\textwidth]{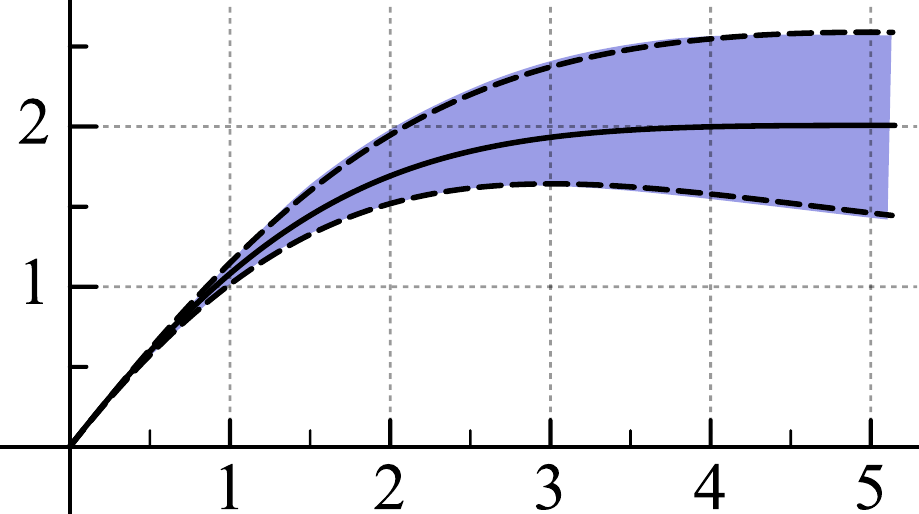}
    \caption{\label{fig:v1}  SFA}
    \end{subfigure}
    \hspace{1px}
    \begin{subfigure}[b]{0.32\textwidth}
    \includegraphics[width=\textwidth]{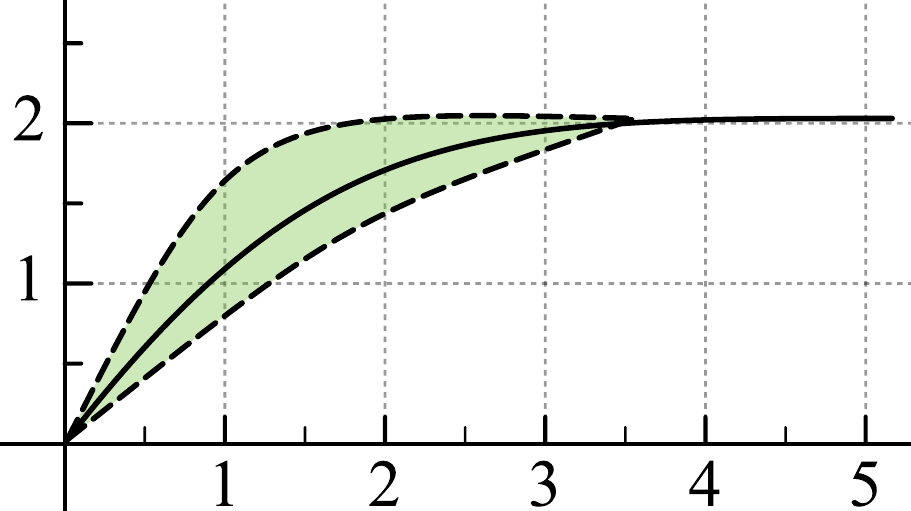}
    \caption{\label{fig:v2} MaxExp(F)}
    \end{subfigure}
    \hspace{1px}
    \begin{subfigure}[b]{0.32\textwidth}
    \includegraphics[width=\textwidth]{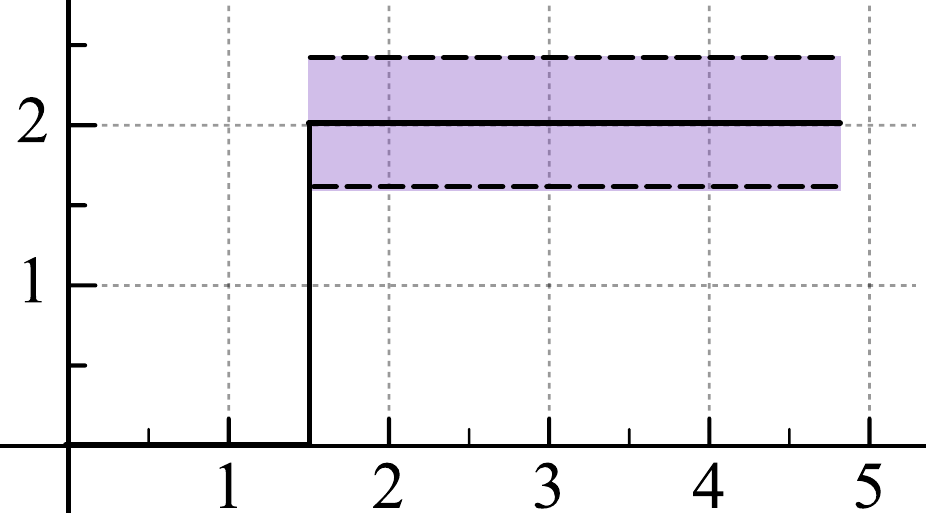}
    \caption{\label{fig:v3} Grassman}
\end{subfigure}
\vspace{-9px}
\vspace{-0.2cm}
\caption{Spectrum augmentation variants (dashed lines: injected noise var.) Power  Norm. \& Power Norm.$^*$ resemble (b) and (a).}\label{fig:new_var}
\end{minipage}
\vspace{-0.3cm}
\end{figure}

\noindent\textbf{Robustness to Noisy Features.}
Below we check on Cora and CiteSeer if GCL with SFA is robust to noisy node features in node classification setting. 
We draw noise $\Delta\mathbf{x} \sim \mathcal{N}(0,\mathbf{I})$ and inject it into the original node features as   $\mathbf{x}+\epsilon\Delta\mathbf{x}$, where $\epsilon\!\geq\!0$ controls the noise intensity.  Table~\ref{tab:nosiy_feature} shows that  SFA is robust to the noisy features. SFA partially balances the spectrum (Fig~\ref{fig:sim2} and~\ref{fig:simulation} ) of leading singular components where the signal lies, while non-leading components where the noise resides are left mostly unchanged by SFA.  

\vspace{0.1cm}
\noindent\textbf{Other Variants of Spectral Augmentation.} Below, we experiment with other ways of performing spectral augmentation. Figure \ref{fig:new_var} shows three different push-forward functions: our SFA, MaxExp(F) \cite{maxexp} and Grassman feature maps   \cite{mehrtash_dict_manifold}. As MaxExp(F) and Grassman do not inject any spectral noise, we equip them with explicit noise injectors. To determine what is the best form of such an operator, we vary (i) where the noise is injected, (ii) profile of balancing curve. As  the push-forward profiles of SFA and MaxExp(F) are similar (Figure \ref{fig:new_var}), we implicitly inject the noise into MaxExp(F) along the non-leading singular values (\cf leading singular values of SFA).

\vspace{0.1cm}
\noindent{\em MaxExp(F)} is only defined for symmetric-positive definite matrices. We extend it to feature maps by:

\vspace{-0.3cm}
\begin{equation}
\begin{aligned}
         \widetilde{\mathbf{H}} \!=\! \mathbf{H}\mathbf{B}\big(\mathbf{I}-(\mathbf{I}-\mathbf{A}/\text{Tr}(\mathbf{A}))^{\eta+\Delta\eta}\big),
    \end{aligned}
    \label{eq:maxexp}
\end{equation}
where $\mathbf{A}=(\mathbf{H}^\top\mathbf{H})^{0.5}$ and $\mathbf{B}=(\mathbf{H}^\top\mathbf{H})^{-0.5}$ are obtained via highly-efficient Newton-Schulz iteration. We draw the random noise  $\Delta\eta\!\sim\!\alpha\mathcal{N}(0,1)$ for each $\mathbf{H}$ and control the variance by $\alpha\!>\!0$. We round $\eta+\Delta\eta$ to the nearest integer (MaxExp(F) requires an integer for fast matrix exponentiation) and ensure the integer is greater than zero. See  the \nameref{sec:maxexp}~section for details.

\vspace{0.1cm}
\noindent
{\em{Matrix Square Root}} is similar to MaxExp(F)  but the soft function flattening the spectrum is replaced with a matrix square root push-forward function. {\em Power Norm.} yields:
%
%

\vspace{-0.3cm}
\begin{equation}
\begin{aligned}
\widetilde{\mathbf{H}}&=\mathbf{H}\mathbf{B}\,\text{PowerNorm}(\mathbf{A}; \beta\!+\!\Delta{\beta}),
\end{aligned}
\end{equation}
where $0\leq\beta\!+\Delta{\beta}\!\leq 1$ controls the steepness of the balancing push-forward curve. $\Delta{\beta}$ is 
drawn from the Normal distribution (we choose the best aug. variance). {\em Power Norm.$^*$} is:
%
%
\vspace{-0.3cm}
\begin{align}
\!\!\!\!\widetilde{\mathbf{H}}&=\mathbf{H}\mathbf{B}\big( (1-\beta\!-\!\Delta{\beta})\mathbf{A}^{0.5} +(\beta\!+\!\Delta{\beta})\mathbf{A}^{0.5}\mathbf{A}^{0.25}\big)\nonumber\\
&=\mathbf{H}\mathbf{B}\,\text{PowerNorm}(\mathbf{A}; \beta\!+\!\Delta{\beta}),
\end{align}
\vspace{-0.3cm}

\noindent
where $0\!\leq\!1\!+\Delta{\beta}^*\!\leq\!2$,  $\Delta{\beta}^*$ is drawn  from the Normal dist. 
%
%
%
%
 $\mathbf{A}^{0.5}$ and $\mathbf{A}^{0.25}\!=\!\left(\mathbf{A}^{0.5}\right)^{0.5}\!$ are obtained by Newton-Schulz iterations. See  \nameref{sec:pnnoise} for details.

\vspace{0.1cm}
\noindent{\em Grassman feature maps } can be obtained as:

\vspace{-0.3cm}
\begin{equation}
\begin{aligned}
         \widetilde{\mathbf{H}} \!=\! \mathbf{H}\mathbf{B}\big(\mathbf{V}\text{Flat}(\mathbf{\Lambda}; \kappa)\mathbf{V}^\top\big),
    \end{aligned}
    \label{eq:grass}
\end{equation}
where $\text{Flat}(\cdot; \kappa)$ sets $\kappa\!>\!0$ leading singular values to $1+\Delta\boldsymbol{\kappa}$ (the rest is 0), and $\Delta\boldsymbol{\kappa}\!\sim\!\alpha\mathcal{N}(0,\mathbf{I})$. In experiments we use SVD and random SVD (injects itself some noise) to produce $\mathbf{V}$ and $\mathbf{\Lambda}$ from $\mathbf{A}$. See  \nameref{sec:grass} for details.

Table \ref{tab:spec_aug} shows that SFA outperforms MaxExp(F), Power Norm., Power Norm.$^*$ and Grassman. As SFA  augments parts of spectrum where signal resides (leading singular values) which is better than  augmenting non-leading singular values  (some noise may reside there) as in MaxExp(F). As Grassman  binarizes spectrum, it may reject some useful signal at the boundary between leading and non-leading singular values. Finally,  \nameref{sec:mp} model reduces the spectral gap, thus rebalances the spectrum (also non-leading part).

%


\begin{table}[t]
\centering
\begin{minipage}[c]{0.42\textwidth}
\centering
    \resizebox{1\textwidth}{!}{
    \hspace{0.05cm}
    \begin{tabular}{cccc}
    \toprule
    \textbf{Power Iteration} & \textbf{Am-Computer} & \textbf{Cora} & \textbf{CiteSeer} \\
    \midrule 
    $k=0$ &$87.25$  &$83.04$   &$71.35$\\
    \rowcolor{LightCyan}$k=1$ &$\mathbf{\textcolor{blue}{ 88.83}}$    &$\textcolor{blue}{\mathbf{85.90}}$   &$73.56$\\
    $k=2$ &$\textcolor{blue}{88.72}$  &$\textcolor{blue}{85.59}$   &$\textcolor{blue}{\mathbf{75.3}}$\\
    $k=4$ &$88.64$  &$85.29$   &${74.8}$\\
    $k=8$ &$88.27$  &$85.23$   &$\textcolor{blue}{74.9}$\\
    \bottomrule
    \end{tabular}
    }
    \vspace{-8px}
    \captionof{table}{\label{tab:EffectofK} Results on different $k$ of SFA.}
    \vspace{5px}
        \centering
    \resizebox{1\textwidth}{!}{
     \hspace{0.05cm}
    \begin{tabular}{l|lcccc}
        \toprule
      $\epsilon$   & & $10^{-4}$  &    $10^{-3}$ &    $10^{-2}$  & $10^{-1}$\\
          \midrule
     \multirow{ 1}{*}{Cora} & w/o SFA       &     80.55        &       70.67           &   62.85       &     59.23  \\ 
     \rowcolor{LightCyan} & with SFA         &       83.14        &       76.56           &     69.59        &     62.55 \\
      \midrule
      \multirow{ 1}{*}{CiteSeer} & w/o SFA         &      67.48        &      62.32           &     52.69        &    42.81 \\
     \rowcolor{LightCyan} & with SFA       &     72.12        &       70.36          &  60.40       &     46.44 \\ 
     \bottomrule
        \end{tabular}
        }
        \vspace{-8px}
        \captionof{table}{Results on noisy features.\label{tab:nosiy_feature}}
    \vspace{5px}
     \resizebox{\textwidth}{!}{
    \begin{tabular}{l|ccc}
    \toprule
    \textbf{Spectrum Aug}. & \textbf{Am-Comput.} & \textbf{Cora} & \textbf{CiteSeer} \\
    \midrule 
     \rowcolor{LightCyan}SFA (ours)  &$\mathbf{\textcolor{blue}{ 88.83}}$    &$\textcolor{blue}{\mathbf{85.90}}$   &$\textcolor{blue}{\mathbf{75.3}}$\\
    \midrule
    MaxExp(F)  &  $87.13$  & $84.19$   & $72.85$\\
    MaxExp(F) (w/o noise)        &  $86.83$  & $83.52$   & $72.35$\\
\midrule
     Power Norm.$^*$ & 86.7& 82.9      &  70.4 \\
     Power Norm. & 86.6& 82.7      &  70.2 \\
     Power Norm. (w/o noise) & 86.3 & 82.5      &  69.9 \\
    \midrule
     Grassman         &  $86.23$  & $82.57$   & $70.15$\\
    Grassman (w/o noise)                &  $85.83$  & $81.17$   & $69.85$\\
    Grassman (rand. SVD)                &  $85.33$  & $81.01$   & $69.95$\\
    \midrule
    Matrix Precond.    & $82.52$ & $78.14 $    &$67.73$     \\
    \bottomrule
    \end{tabular}
    }
    \vspace{-8px}
    \caption{Results of various spectral augmentations.$\!\!\!\!$\label{tab:spec_aug}}
    \vspace{-0.3cm}
\end{minipage}
\end{table}

\section{Conclusions}
We have shown that GCL is not restricted to only link perturbations or feature augmentation.
By introducing a simple and efficient spectral feature augmentation layer 
we achieve significant performance gains. 
Our incomplete power iteration is very fast. 
%
%
Our theoretical analysis has demonstrated that SFA rebalances the useful part of spectrum, and also augments the useful part of spectrum by implicitly injecting the noise into singular values of both data views. SFA leads to  a better alignment with a lower generalization bound. 
%
%

\newpage
\section{Acknowledgments}
We thank anonymous reviewers for their valuable comments. 
The work described here was partially supported by grants from the National Key Research and Development Program of China (No. 2018AAA0100204) and from the Research Grants Council of the Hong Kong Special Administrative Region, China (CUHK 2410021, Research Impact Fund, No. R5034-18).
\bibliography{aaai23}

\newpage
\clearpage
\appendix
\title{Spectral Feature Augmentation for Graph Contrastive Learning and Beyond (Supplementary Material)}
\author{
    Yifei Zhang$^{1}$, Hao Zhu$^{2,3}$, Zixing Song$^{1}$, Piotr Koniusz$^{3,2,*}$, Irwin King$^{1}$\\
}

\let\titlearea\relax
\let\actualheight\relax


\maketitle


\setcounter{footnote}{4}

\section{Notations}
\label{sec:not}
In this paper, a graph with node features is denoted as $\mathcal{G}=(\mathbf{X}, \mathbf{A})$, where $\mathcal{V}$ is the vertex set, $\mathcal{E}$ is the edge set, and $\mathbf{X} \in \mathbb{R}^{n \times d_{x}}$ is the feature matrix (\ie, the $i$-th row of $\mathbf{X}$ is the feature vector $\mathbf{x}_i$ of node $v_{i}$) and $\mathbf{A} \in\{0,1\}^{n \times n}$ denotes the adjacency matrix of $G$, \ie, the $(i, j)$-th entry in $\mathbf{A}$ is 1 if there is an edge between $v_{i}$ and $v_{j}$.
The degree of node $v_{i}$, denoted as $d_{i}$, is the number of edges incident with $v_{i}$.
The degree matrix $\mathbf{D}$ is a diagonal matrix and its $i$-th diagonal entry is $d_{i}$. 
For a $d$-dimensional vector $\mathbf{x} \in \mathbb{R}^{d}, \|\mathbf{x}\|_{2}$ is the Euclidean norm of $\mathbf{x}$.
We use ${x}_{i}$ to denote the $i$ th entry of $\mathbf{x}$, and $\operatorname{diag}(\mathbf{x}) \in \mathbb{R}^{d \times d}$ is a diagonal matrix such that the $i$-th diagonal entry is $x_{i}$. 
We use $\mathbf{a}_{i}$ denote the row vector of $\mathbf{A}$ and ${a}_{ij}$ for the $(i, j)$-th entry of $\mathbf{A}$. 
The trace of a square matrix $\mathbf{A}$ is denoted by $\operatorname{Tr}(\mathbf{A})$, which is the sum along the diagonal of $\mathbf{A}$. 
%
%
We use $\| \mathbf{A} \|_2$ to denote the spectral norm of $\mathbf{A}$, which is its largest singular value $\sigma_{\max}$. 
We use $\|\mathbf{A}\|_F$ for the Frobenius Norm, which is $\|\mathbf{A}\|_{\mathrm{F}}=\sqrt{\sum_{i,j}\left|a_{i j}\right|^{2}}=\sqrt{\operatorname{
Tr}\left(\mathbf{A^\top A}\right)}$.

\section{Bounded Noise Matrix.}
For completeness, we  analyze the relation between the orthonormal bases $\  {\mathbf{V}}$ and $\mathbf{V}$ of   $\widetilde{\mathbf{H}}$ and $\mathbf{H}$, the augmentation error bounded by $\eta$, and the distribution of the spectrum. Let $\mathbf{E} = |\widetilde{\mathbf{H}} - \mathbf{H}|$ be  the absolute error matrix where $\mathbf{E}_{ij} = |\widetilde{\mathbf{H}}_{ij} - \mathbf{H}_{ij}| \leq \eta$. Let $\{\widetilde{\sigma}_i\}$ and $\{\sigma_i\}$ be sets of singular values of $\widetilde{\mathbf{H}}$ and $\mathbf{H}$. We derive the following bound.
\begin{proposition}
\label{prop:ebound}
If 
$\|\mathbf{V}\!-\!\widetilde{\mathbf{V}}\|_2\!\leq\!\epsilon$,  each element in $\mathbf{E}_{ij}$ (the absolute error) is bounded by 
$\eta = 2 \varepsilon\Delta\tilde{\sigma}_{12} /(n\pi+2\varepsilon)$, where $n$ is the number of nodes and $\Delta\tilde{\sigma}_{12}=\widetilde{\sigma}_1-\widetilde{\sigma}_2$ is the singular value gap  of $\widetilde{\mathbf{H}}$.
\end{proposition}
\vspace{-1cm}
\begin{proof}
See \nameref{proof:ebound} (suppl. material).
\end{proof}


\definecolor{beaublue}{rgb}{0.88,1,1}
\definecolor{blackish}{rgb}{0.2, 0.2, 0.2}
\begin{tcolorbox}[width=1.0\linewidth, colframe=blackish, colback=beaublue, boxsep=0mm, arc=2mm, left=2mm, right=2mm, top=2mm, bottom=2mm]
 Prop. \ref{prop:ebound} shows that the absolute difference between the augmented and  original feature matrices is bounded (variance in Prop. \ref{prop:PowerIterVar} does not guarantee that). Moreover, as Prop. \ref{prop:PowerIter1} shows the flattening of spectrum indeed takes place, this means that the spectral gaps $\Delta\tilde{\sigma}^\alpha_{12}$ and $\Delta\tilde{\sigma}^\beta_{12}$ of both augmented feature matrices   $\widetilde{\mathbf{H}}^\alpha$ and $\widetilde{\mathbf{H}}^\beta$ from both views are bounded and limited compared to the spectral gaps of original matrices ${\mathbf{H}}^\alpha$ and ${\mathbf{H}}^\beta$. The profound consequence of this observation is that $\widetilde{\mathbf{V}}^\alpha$ and $\widetilde{\mathbf{V}}^\beta$ are then closer to original ${\mathbf{V}}^\alpha$ and ${\mathbf{V}}^\beta$, and so the angle alignment between of two views enjoys an improved quality.  
%
\end{tcolorbox}

\section{Proofs of Propositions}

Below we present proofs for several propositions that were not included in the main draft. 

\setcounter{subsection}{1}
\subsection{Proof of Proposition \ref{prop:FeatureAug}}

\begin{proof}
Let $\mathbf{H} = \mathbf{U} \boldsymbol{\Sigma} \mathbf{V}^{\top}$ be the SVD of an input feature map, where $\mathbf{U}$ and $\mathbf{V}$ are the right and left orthonormal matrices that contain the singular vectors. As $\mathbf{r}^{(k)} = (\mathbf{H}^\top\mathbf{H})^k\mathbf{r}^{(0)}$, we have: 
\begin{equation}
    \mathbf{r}^{(k)} = (\mathbf{H}^\top\mathbf{H})^{k}\mathbf{r}^{(0)} = (\mathbf{V}\boldsymbol{\Sigma}^{2k} \mathbf{V}^{\top})\mathbf{r}^{(0)}.
    \label{equ:Iter}
\end{equation}
Plugging Equation~(\ref{equ:Iter}) into Equation~(\ref{equ:FeatureAug}), we obtain:
\begin{equation}
\begin{aligned}
        \!\!\!\!\frac{\mathbf{H} \mathbf{r}^{(k)} \mathbf{r}^{(k)\top}}{\|\mathbf{r}^{(k)}\|_2^2} &\!=\! \frac{\mathbf{U}\Sigma\mathbf{V}^\top (\mathbf{V}\boldsymbol{\Sigma}^{2k} \mathbf{V}^{\top})\mathbf{r}^{(0)} \mathbf{r}^{(0)\top}(\mathbf{V}\boldsymbol{\Sigma}^{2k} \mathbf{V}^{\top})}{\mathbf{r}^{0}(\mathbf{V}\boldsymbol{\Sigma}^{2k} \mathbf{V}^{\top})(\mathbf{V}\boldsymbol{\Sigma}^{2k} \mathbf{V}^{\top})\mathbf{r}^{0}}\\
        &\!=\! \frac{\mathbf{U}\boldsymbol{\Sigma}^{2k+1} \mathbf{V}^{\top}\mathbf{r}^{(0)} \mathbf{r}^{(0)\top}\mathbf{V}\boldsymbol{\Sigma}^{2k} \mathbf{V}^{\top}}{\mathbf{r}^{(0)\top}\mathbf{V} \Sigma^{4k} \mathbf{V}^{\top}\mathbf{r}^{(0)}}.\!\!
        \label{eq:com}
\end{aligned}
\end{equation}
Let $\mathbf{y} = \mathbf{V}^{\top}\mathbf{r}^{(0)}$ and $k\geq1$ be the number of iterations, then we simplify Eq.~\eqref{eq:com} as:
\begin{align}
        &\frac{\mathbf{H} \mathbf{r}^{(k)} \mathbf{r}^{(k)\top}}{\|\mathbf{r}^{(k)}\|_2^2} \!=\!\mathbf{U}\mathbf{Q}(\mathbf{y},k)\mathbf{V}^\top\\
        \text{ where }\;& \mathbf{Q}(\mathbf{y},k)\!=\!\frac{\boldsymbol{\Sigma}^{2k+1}\mathbf{y}\mathbf{y}^\top\boldsymbol{\Sigma}^{2k}}{\mathbf{y}^\top\boldsymbol{\Sigma}^{4k}\mathbf{y}}.
        \label{eq:sim}
\end{align}
As $\mathbf{r}^{(0)} \sim \mathcal{N}(0,\mathbf{I})$ and $\mathbf{V}$ is a unitary matrix, we have $\mathbf{y}=[y_1,\cdots,y_{d_h}]\sim\mathcal{N}(0, \mathbf{I})$. The expectation of $\mathbf{Q}$ 
is a diagonal matrix, \ie, we have $\mathbb{E}_{\mathbf{y}\sim\mathcal{N}(0, \mathbf{I})}(\mathbf{Q}(\mathbf{y},k)) = \text{diag}\big(q_1(k), q_2(k), \cdots, q(k)_{d_h}\big)$ with:

\vspace{-0.3cm}
\begin{equation}
\begin{aligned}
\label{eq:balanceq}
    q_{i}(k)=\sigma_{i} \lambda_{i}(k)  &\text{ where }  \lambda_{i}(k)=\mathbb{E}_{\mathbf{y}\sim\mathcal{N}(0;\mathbf{I})}\big(\lambda'_i(\mathbf{y},k)\big), \; \\
    &{\text{ and }} \lambda'_i(\mathbf{y},k)=
    \text{\fontsize{8}{8}\selectfont$\frac{\left(y_{i} \sigma_{i}^{2k}\right)^{2}}{ \sum_{l=1}^{d_h}\left(y_{l} \sigma_{l}^{2k}\right)^{2}}$}.
    \end{aligned}
\end{equation}
For brevity, let us drop parameter $k$ where possible. We need to show that $1 \geq \lambda_i \geq \lambda_j\geq 0$ when $\sigma_i \geq \sigma_j$. Obviously, $1 \geq\lambda_i\geq0$, thus we need to show that $\lambda_i-\lambda_j\geq 0$ (note\footnote{Each expectation in Equation \eqref{eq:balance} runs over ${\mathbf{y}\sim\mathcal{N}(0;\mathbf{I})}$, that is, $\mathbb{E}_{\mathbf{y}\sim\mathcal{N}(0;\mathbf{I})}(\cdot)$.}):
\begin{equation}
\label{eq:balance}
\begin{aligned}
          \lambda_i - \lambda_j =& \mathbb{E}\left(\frac{\left(y_{i} \sigma_{i}^{2k}\right)^{2}}{ \sum_{l=1}^{d_h}\left(y_{l} \sigma_{l}^{2k}\right)^{2}}\right)-\mathbb{E}
          \Bigg(\frac{\big(y_{j} \sigma_{j}^{2k}\big)^{2}}{ \sum_{l=1}^{d_h}\left(y_{l}\sigma_{l}^{2k}\right)^{2}}\Bigg)\\
          &\geq \sigma_{j}^{4k} \mathbb{E}\left(\frac{y^{2}_i - y^{2}_j}{ \sum_{l=1}^{d_h}\left(y_{l} \sigma_{l}^{2k}\right)^{2}}\right) = 0.
\end{aligned}
\end{equation}

As $$\mathbb{E}_{\mathbf{y}\sim\mathcal{N}(0;\mathbf{I})}\!\left(\frac{y^{2}_i - y^{2}_j}{ \sum_{l=1}^{d_h}\left(y_{l} \sigma_{l}^{2k}\right)^{2}}\right)=$$ 
 $$\mathbb{E}_{\mathbf{y}\sim\mathcal{N}(0;\mathbf{I})}\!\left(\frac{y^{2}_i}{ \sum_{l=1}^{d_h}\left(y_{l} \sigma_{l}^{2k}\right)^{2}}\right)\!-\!
  \mathbb{E}_{\mathbf{y}\sim\mathcal{N}(0;\mathbf{I})}\!\left(\frac{y^{2}_j}{ \sum_{l=1}^{d_h}\left(y_{l} \sigma_{l}^{2k}\right)^{2}}\right)
\!=\!0$$ 

\noindent
due to
  $y_i$, $y_j$ being i.i.d. random variables sampled from the normal distribution,  inequality \eqref{eq:balance} holds. 
\end{proof}

\subsection{Proof of Proposition \ref{prop:PowerIter1}}
\label{exp_val}
\begin{proof}[Proof]
We split the proof into four steps as follows.

\vspace{0.1cm}\noindent
\textbf{Step 1.} 
Let $\beta_i=(\sigma_i^{2k})^2$.  Define random variable $x_i(\mathbf{y})=\frac{\beta_i y_{i}^2}{\beta_i y_{i}^2 +\sum_{l\neq i}\beta_l y^2_{l} }=\frac{y_{i}^2}{y_{i}^2 +\sum_{l\neq i}\frac{\beta_l}{\beta_i}y^2_{l} }=\frac{u}{u+v_i}$ where ${\mathbf{y}\sim\mathcal{N}(0;\mathbf{I})}$, $u=y_i^2$ and $v_i=\sum_{l\neq i}\frac{\beta_l}{\beta_i}y^2_{l}$. Notice that transformation $y_i^2$ of the random variable ${{y_i}\sim\mathcal{N}(0;1)}$ in fact $\chi^2(1)$ distributed, \ie, $y_i^2\sim\chi^2(1)=\mathcal{G}(\frac{1}{2},2)$ because $\chi^2(t)\equiv\mathcal{G}(\frac{t}{2},2)$ (it is a well-known fact that the $\chi^2(t)$ is a special case of the Gamma distribution defined with the scale parameter).

\vspace{0.1cm}\noindent
\textbf{Step 2.} 
Next, finding the exact distribution of $v_i$ (a weighted sum of chi squares) is a topic of ongoing studies but highly accurate surrogates exist, \ie, authors of \citelatex{weighted_chi_squares_sup} show that   $\sum_l\zeta_l y_l^2\sim\mathcal{G}\Big(\frac{1}{2}\frac{(\sum_l\zeta_l)^2}{\sum_l\zeta_l^2}, 2\frac{\sum_l\zeta_l^2}{\sum_l \zeta_l} \Big)$ for weights $\zeta_l\geq 0$ and $\sum_l \zeta_l>0$, and so we have 
$v_i\sim\mathcal{G}\Big(\frac{1}{2}\frac{(\sum_{l\neq i}\beta_l)^2}{\sum_{l\neq i}\beta_l^2}, \frac{2}{\beta_i}\frac{\sum_{l\neq i}\beta_l^2}{\sum_{l\neq i} \zeta_l} \Big)$.

\vspace{0.1cm}\noindent
\textbf{Step 3.} 
Now, our proof requires determining the distribution of $x_i=\frac{u}{u+v_i}$ where $u\sim\mathcal{G}(\alpha_0,\theta_0)=\mathcal{G}(\frac{1}{2},2)$ and $v_i\sim\mathcal{G}(\alpha_i,\theta_i)=\mathcal{G}\Big(\frac{1}{2}\frac{(\sum_{l\neq i}\beta_l)^2}{\sum_{l\neq i}\beta_l^2}, \frac{2}{\beta_i}\frac{\sum_{l\neq i}\beta_l^2}{\sum_{l\neq i} \zeta_l} \Big)$. Let $x=\frac{u}{u+v}=\frac{1}{1+b}$ for $u\sim\mathcal{G}(\alpha_0,\theta)$ and $v\sim\mathcal{G}(\alpha_i,\theta)$, then  the following are well-known results that  $x\sim\mathcal{B}(\alpha_0,\alpha_i)$  and $b=u/v\sim\mathcal{B}'(\alpha_i,\alpha_0)$ where $\mathcal{B}$ and $\mathcal{B}'$ stand for the Beta distribution and the Standard Beta Prime distribution, respectively. Let $x_i=\frac{u}{u+v_i}=\frac{1}{1+\gamma_i b_i},\;0\leq\gamma_i\leq\infty$ where $\gamma_i=\frac{\theta_i}{\theta_0}$ (ratio of scalars in Gamma distributions, \ie, $v_i\sim\mathcal{G}(\alpha_i,\theta_i)$)  and $b_i=\frac{1-x_i}{x_i}$ (solving $x_i=\frac{1}{1+b_i}$ for $b_i$). Substituting $b_i=\frac{1-x_i}{x_i}$ into $\frac{1}{1+\gamma_i b_i}$ yields $z=\frac{x_i}{\gamma_i+(1-\gamma_i)x_i}$ and since we have to perform the change of variable, we compute (i) $x_i=\frac{\gamma_i z}{1-(1-\gamma_i)z}$ and (ii) the Jacobian which takes a simple form $\frac{\partial x_i}{\partial z}=\frac{\gamma_i}{(\gamma_i+(1-\gamma_i)z)^2}$. Finally, we have $x^\bullet_i(z)= \mathcal{B}(x_i; \alpha_0,\alpha_i)\cdot\lvert\frac{\partial x_i}{\partial z}\rvert=\frac{\gamma_i}{(\gamma_i+(1-\gamma_i)z)^2}\cdot\mathcal{B}\left(\frac{\gamma_i z}{1-(1-\gamma_i)z}; \alpha_0,\alpha_i\right)$. As $\alpha_0=\frac{1}{2}$ and $\theta_0=2$, we have $\gamma_i=\frac{1}{\beta_i}\frac{\sum_{l\neq i}\beta_l^2}{\sum_{l\neq i} \zeta_l} $, $\theta_i=2\gamma_i$ (following our earlier assumptions on modeling $v_i$), and the PDF underlying our spectrum-balancing algorithm is $x^\bullet_i(z)=\frac{\gamma_i}{(\gamma_i+(1-\gamma_i)z)^2}\cdot\mathcal{B}\Big(\frac{\gamma_i z}{1-(1-\gamma_i)z}; \frac{1}{2},\alpha_i\Big)$. Moreover, the above PDF enjoys the support $z\in[0 ;1]$ because $\mathcal{B}(x_i)$ enjoys the support $x_i\in[0 ;1]$ and function $x_i: [0 ;1]\longmapsto[0 ;1]$.

\vspace{0.1cm}\noindent
\textbf{Step 4.} 
Finally, $\lambda_i=\mathbb{E}(z)=\int_0^1 z\cdot x^\bullet_i(z) \, \mathrm{d}z$ is simply obtained by the integration using the Mathematica software followed by a few of algebraic simplifications regarding switching the so-called regularized Hypergeometric function with the so-called Hypergeometric function.
\end{proof}

\subsection{Proof of Proposition \ref{prop:PowerIterVar}}
\label{exp_var}
\begin{proof}[Proof]
Variance $\omega^2_i=\mathbb{E}(z^2)-(\mathbb{E}(z))^2=\int_0^1 z^2\cdot x^\bullet_i(z) \, \mathrm{d}z - \lambda^2_i$ relies on the PDF expression $x^\bullet_i(z)$ derived in the Proof \ref{exp_val}. Expression $\int_0^1 z^2\cdot x^\bullet_i(z) \, \mathrm{d}z$ is simply obtained by the integration using the Mathematica software  followed by a few of algebraic simplifications regarding switching the so-called regularized Hypergeometric function with the so-called Hypergeometric function.
\end{proof}

\subsection{Proof of Proposition~\ref{prop:ebound}}
\label{proof:ebound}
\begin{proof}[Proof]
Since $\|\mathbf{E}\|_2^2 \leq \|\mathbf{E}\|^2_F \leq n^2 \tau^2$, we have:
\begin{equation}
\label{equ:nbequ0}
    \|\mathbf{E}\|_2 \leq n\tau.
\end{equation}
From works~\citelatex{kostrykin2003subspace_sup,wang2008manifold_sup}, we know that if $\widetilde{\mathbf{H}}$ and $\mathbf{E}$ are self-adjoint operators on a separable Hilbert space, then the spectrum of $\mathbf{H}$ is in the closed $\|\mathbf{E}\|-$neighborhood of the spectrum of $\widetilde{\mathbf{H}}$. Thus, we have the following inequality:
\begin{equation}
\label{equ:nbequ}
    \|\mathbf{V}^{\perp} \mathbf{V}'\|_2 \leq \pi\|\mathbf{E}\|_2/2\Delta\tilde{\sigma},
\end{equation}
As we know $\widetilde{\mathbf{H}}$ has the singular value gap $\Delta\tilde{\sigma}_{12}$
We need to assume $\|\mathbf{E}\|_2<\Delta\tilde{\sigma} / 2$ to guarantee $\mathbf{H}$ also has the singular value gap. Thus, Equation~(\ref{equ:nbequ}) becomes:
\begin{equation}
\label{equ:nbequ2}
    \|\mathbf{V}^{\perp} \widetilde{\mathbf{V}}\|_2 \leq \pi\|\mathbf{E}\|_2/2(\Delta\tilde{\sigma}_{12}-\|\mathbf{E}\|_2).
\end{equation}
Similarly, we also have:
\begin{equation}
\label{equ:nbequ3}
    \|\mathbf{V}\widetilde{\mathbf{V}}^{\perp}\|_2 \leq \pi\|\mathbf{E}\|_2/2(\Delta\tilde{\sigma}_{12}-\|\mathbf{E}\|_2).
\end{equation}
Since $\|\mathbf{V}-\widetilde{\mathbf{V}}\|_2 = \max(\|\mathbf{V}\widetilde{\mathbf{V}}^{\perp}\|_2,\|\mathbf{V}^{\perp}\widetilde{\mathbf{V}}\|_2)$, we have the following bound:
\begin{equation}
\label{equ:bound}
    \|\mathbf{V}-\widetilde{\mathbf{V}}\|_2  \leq \pi\|\mathbf{E}\|_2/2(\Delta\tilde{\sigma}_{12}-\|\mathbf{E}\|_2).
\end{equation}
Equation~(\ref{equ:bound}) implies if $\|\widetilde{\mathbf{V}} - \mathbf{V}\|_2<\epsilon$, we get $\|\mathbf{E}\|_2 \leq 2 \Delta\tilde{\sigma}_{12} \varepsilon/(2 \varepsilon+\pi)$. Combined with Equation~(\ref{equ:nbequ0}), we arrive at the conclusion that if the difference of $\|\mathbf{V}-\widetilde{\mathbf{V}}\|_2$ is at most $\epsilon$, the absolute error of each element in $\mathbf{E}$ is bounded by $\tau$ and $\tau \leq 2 \varepsilon \Delta\tilde{\sigma}_{12} /(n(\pi+2 \varepsilon))$.
\end{proof}

\subsection{Proof of Theorem~\ref{theorem:gb}}
\label{gen_bound_proof}
\begin{proof}[Proof]
\begin{definition}~\citelatex{DBLP:journals/corr/abs-2111-00743_sup}
 $((\sigma, \delta)$-Augmentation). The data augmentation set $\mathcal{A}$ is called a $(\sigma, \delta)$-augmentation, if for each class $C_{k}$, there exists a subset $C_{k}^{0} \subseteq C_{k}$ (called the main part of $C_{k}$ ) such that the following two conditions hold: (i) $\mathbb{P}\left[\mathbf{x} \in C_{k}^{0}\right] \geq \sigma \mathbb{P}\left[\mathbf{x} \in C_{k}\right]$ where $\sigma \in(0,1]$, (ii) $\sup _{\mathbf{x}_{1}, \mathbf{x}_{2} \in C_{k}^{0}} d_\mathcal{A}\left(\mathbf{x}_{1}, \mathbf{x}_{2}\right) \leq \delta$.
\end{definition}

\begin{lemma}~\citelatex{DBLP:journals/corr/abs-2111-00743_sup}
For a $(\sigma, \delta)$-augmentation with main part $C_{k}^{0}$ of each class $C_{k}$, if all samples belonging to $\left(C_{1}^{0} \cup \cdots \cup C_{K}^{0}\right) \cap S_{\varepsilon}$ can be correctly classified by a classifier $G$, then its downstream error rate $\operatorname{Err}(G) \leq$ $(1-\sigma)+R_{\varepsilon}$. $R_{\varepsilon}:=\mathbb{P}\left[{S_{\varepsilon}}\right]$ and $S_{\varepsilon}:=\{\mathbf{x} \in \cup_{k=1}^{K} C_{k}: \forall \mathbf{x}_{1}, \mathbf{x}_{2} \in \mathcal{A}(\mathbf{x}),\left\|f\left(\mathbf{x}_{1}\right)-f\left(\mathbf{x}_{2}\right)\right\| \geq \varepsilon\}$.
\end{lemma}
Because $\left\|f\left(\mathbf{x}_{1}\right)-f\left(\mathbf{x}_{2}\right)\right\|$ is a non-negative random variable. Then, for any $\epsilon>0$, according to markov inequality we have:
\begin{equation}
\begin{aligned}
          &\mathbb{P}[\left\|f\left(\mathbf{x}_{1}\right)-f\left(\mathbf{x}_{2}\right)\right\|_2^2 \geq \epsilon^2 ]\\
          &\leq {\mathbb{E}_{\mathbf{x}_{1}, \mathbf{x}_{2} \in \mathcal{A}(\mathbf{x})}[\left\|f\left(\mathbf{x}_{1}\right)-f\left(\mathbf{x}_{2}\right)\right\|_2^2]}/{\epsilon^2}.
\end{aligned}
\end{equation}
Let $f(\mathbf{x})=1$ and  $\mathbb{E}_{\mathbf{x}_{1}, \mathbf{x}_{2} \in \mathcal{A}(\mathbf{x})}[\left\|f\left(\mathbf{x}_{1}\right)-f\left(\mathbf{x}_{2}\right)\right\|_2^2] = 2 - 2\mathbb{E}_{\mathbf{x}_{1}, \mathbf{x}_{2} \in \mathcal{A}(\mathbf{x})}[f\left(\mathbf{x}_{1}\right)^\top f\left(\mathbf{x}_{2}\right)]= 2-2\mathcal{L}_{a}$. And thus we have $R_{\varepsilon}\!\!= \!\!P_{\mathbf{x}_{1}, \mathbf{x}_{2} \in \mathcal{A}(\mathbf{x})}\{\|f\left(\mathbf{x}_{1}\right)\!-\!f\left(\mathbf{x}_{2})\| \!\geq\! \varepsilon\right\}\!\leq\!\frac{\sqrt{2-2\mathcal{L}_{a}}}{\varepsilon}$.
\end{proof}

\subsection{Upper bound of $\phi$}
\label{lambda_upper}
As SFA seeks to remove the contribution of the largest singular value of $\mathbf{H}$ from the spectrum $\sigma_1\geq\sigma_2\geq\cdots\geq\sigma_{d_h}$ of $\mathbf{H}$, it follows that for $\sigma_i$ we have the upper bound on the push-forward function such that $\phi(\sigma_i;k)\leq \bar{\phi}(\sigma_i)-\mathcal{O}(k)$ where $\bar{\phi}(\sigma_i)=\min(\sigma_i, \max\limits_{j\neq i}(\sigma_j))$ and $\mathcal{O}(k)\geq 0$.

 Figure \ref{fig:bound} illustrates the above bound. As is clear, the lower the value $k$ is, the more ample is the possibility to tighten that bound (the gap to tighten is captured by $\mathcal{O}(k)\geq0$).

Nonetheless, with this simple upper bound, we show that we clip/balance spectrum and so the loss in Eq. \eqref{equ:EignAlignRefine} enjoys lower energy than the loss in Eq. \eqref{equ:EignAlignOri}, on which  Theorem~\ref{theorem:gb} relies.

\begin{figure}[h]
\centering

\includegraphics[width=0.3\textwidth]{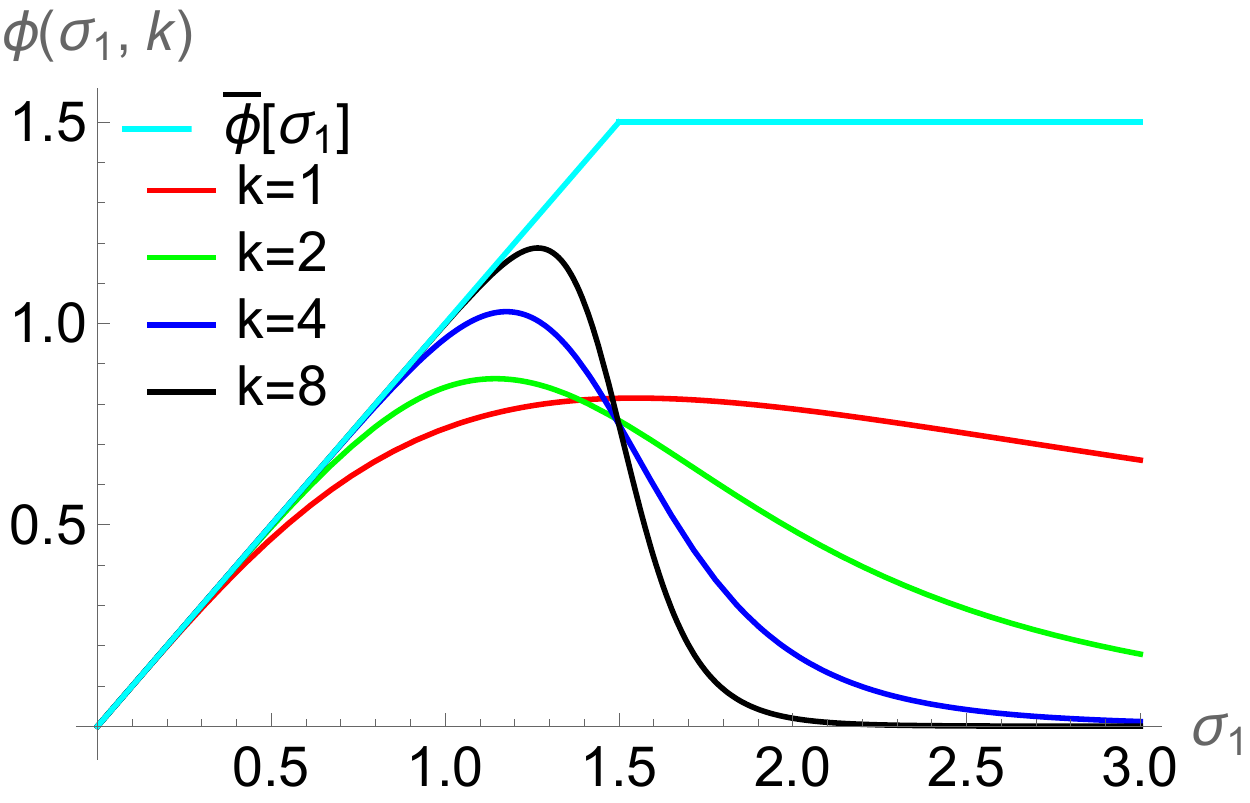}
\caption{Illustration of $\phi(\sigma_1;k)$ \vs $\bar{\phi}(\sigma_1)$.}
\label{fig:bound}
\end{figure}


Notice that the upper bound in Figure \ref{fig:bound} is also known as the so-called AxMin pooling operator \citelatex{s19_sup}.

\subsection{Derivation of MaxExp(F)}
\label{sec:maxexp}

Derivation of MaxExp(F) in Eq. \eqref{eq:maxexp} is based on the following steps. Based on SVD, we define $\mathbf{H}=\mathbf{U}\mathbf{\Sigma}\mathbf{V}^\top$, where $\mathbf{U}$ and $\mathbf{V}$ are left and right singular vectors of $\mathbf{H}$ and $\mathbf{\Sigma}$ are singular values placed along the diagonal. 

\vspace{0.1cm}
\noindent Let $\mathbf{A}^2\!=\!\mathbf{H}^\top\mathbf{H}\!=\!\mathbf{V}\mathbf{\Sigma}^2\mathbf{V}^\top$ and $\mathbf{A}\!=\!\big(\mathbf{H}^\top\mathbf{H})^{0.5}\!=\!\mathbf{V}\mathbf{\Sigma}\mathbf{V}^\top$. Let  $\mathbf{B}^2\!=\!(\mathbf{H}^\top\mathbf{H})^{-1}\!=\!\mathbf{V}\mathbf{\Sigma}^{-2}\mathbf{V}^\top$ and $\mathbf{B}\!=\!\big(\mathbf{H}^\top\mathbf{H}\big)^{-0.5}\!=\!\mathbf{V}\mathbf{\Sigma}^{-1}\mathbf{V}^\top$.

\vspace{0.1cm}
\noindent We apply MaxExp(F) from \citelatex{s17_supp,maxexp_sup,hosvd_sup} to $\mathbf{A}$, that is:
\begin{equation}
\begin{aligned}
\mathbf{C}&=\big(\mathbf{I}-(\mathbf{I}-\mathbf{A}/\text{Tr}(\mathbf{A}))^{\eta+\Delta\eta}\big)\\
&=\mathbf{V}\big(\mathbf{I}-(\mathbf{I}-\mathbf{\Sigma}/\text{Tr}(\mathbf{\Sigma}))^{\eta+\Delta\eta}\big)\mathbf{V}^\top.
\end{aligned}
\end{equation}
Multiplying $\mathbf{C}$ from the left side by $\mathbf{H}\mathbf{B}$, \ie, $\widetilde{\mathbf{H}}=\mathbf{H}\mathbf{B}\mathbf{C}$ yields: 
\begin{equation}
\begin{aligned}
\widetilde{\mathbf{H}}&=\mathbf{U}\mathbf{\Sigma}\mathbf{V}^\top\mathbf{V}\mathbf{\Sigma}^{-1}\mathbf{V}^\top \mathbf{V}\big(\mathbf{I}-(\mathbf{I}-\mathbf{\Sigma}/\text{Tr}(\mathbf{\Sigma}))^{\eta+\Delta\eta}\big)\mathbf{V}^\top\\
&=\mathbf{U}\big(\mathbf{I}-(\mathbf{I}-\mathbf{\Sigma}/\text{Tr}(\mathbf{\Sigma}))^{\eta+\Delta\eta}\big)\mathbf{V}^\top\\
&=\mathbf{H}\mathbf{B}\big(\mathbf{I}-(\mathbf{I}-\mathbf{A}/\text{Tr}(\mathbf{A}))^{\eta+\Delta\eta}\big)\\
&=\mathbf{H}\mathbf{B}\,\text{MaxExp}(\mathbf{A}; \eta+\Delta\eta).
\end{aligned}
\end{equation}

\vspace{0.1cm}
\noindent Matrices $\mathbf{A}=(\mathbf{H}^\top\mathbf{H})^{0.5}$ and $\mathbf{B}=(\mathbf{H}^\top\mathbf{H})^{-0.5}$ are obtained via highly-efficient Newton-Schulz iterations 
\citelatex{higham2008functions_sup,ns_sup} 
in Algorithm \ref{alg:sqrt} given $\zeta=10$ iterations. Figure \ref{fig:v2} shows the push-forward function 
\begin{equation}\text{MaxExp}(\sigma_i; \eta, \Delta\eta)=1-\Big(1-\frac{\sigma_i}{\sum_j\sigma_j}\Big)^{\eta+\Delta\eta},\end{equation} 
whose role is to rebalance the spectrum ($\sigma_i$ are coefficients on the diagonal of $\mathbf{\Sigma}$).

\begin{algorithm}[h]
\caption{Newton-Schulz iterations.}
\label{alg:sqrt}
\textbf{Input}: Symmetric positive-definite matrix $\mathbf{M}=\mathbf{H}^\top\mathbf{H}$, total iterations $\zeta$.\\
\textbf{Output}: Matrices $\mathbf{A}=(\mathbf{H}^\top\mathbf{H})^{0.5}$ and $\mathbf{B}=(\mathbf{H}^\top\mathbf{H})^{-0.5}$.

\begin{algorithmic}[1] 
\STATE $\mathbf{M}'=\mathbf{M}/\text{Tr}(\mathbf{M})$
\STATE $\mathbf{P} = \frac{1}{2}(3\mathbf{I}-\mathbf{M}')$, $\mathbf{Y}_0 = \mathbf{M}'\mathbf{P}$, $\mathbf{Z}_0 = \mathbf{P}$
\FOR{$i=1$ to $\mathrm \zeta$} 
    \STATE $\mathbf{P} = \frac{1}{2}(3\mathbf{I}-\mathbf{Z}_{i-1}\mathbf{Y}_{i-1})$\;
    \STATE $\mathbf{Y}_i = \mathbf{Y}_{i-1}\mathbf{P}$, $\mathbf{Z}_i = \mathbf{P}\mathbf{Z}_{i-1}$\;
\ENDFOR
\STATE $\mathbf{A} = \mathbf{P}\sqrt{\mathrm{Tr}(\mathbf{M})}$ and $\mathbf{B} = \mathbf{Y}_\zeta/\sqrt{\mathrm{Tr}(\mathbf{M})}$. 
\end{algorithmic}
\end{algorithm}

\subsection{Derivation of Grassman}
\label{sec:grass}

Derivation of Grassman feature map in Eq. \eqref{eq:grass} follows the same steps as those described in the \nameref{sec:maxexp} section above. The only difference is that the soft function flattening the spectrum is replaced with its  binarized variant illustrated in Figure \ref{fig:v3} describing a push-forward function:
\begin{equation}
  \text{Flat}(\sigma_i; \kappa, \Delta\boldsymbol{\kappa})=
    \begin{cases}
      1+\Delta\kappa_i & \text{if $i\leq\kappa$}\\
      0 & \text{otherwise},
    \end{cases}
\end{equation}
where $\Delta\kappa_i$ are coefficients of $\Delta\boldsymbol{\kappa}$ random vector drawn from from the Normal distribution (we choose the best augmentation variance) per sample, and $\kappa>0$ controls the spectrum cut-off.
%

\subsection{Derivation of Matrix Square Root.}
\label{sec:pnnoise}
Another approach towards spectrum rebalancing we try is based on derivations in the  \nameref{sec:maxexp} section above. The key difference is that the soft function flattening the spectrum is replaced with a matrix square root push-forward function:
\begin{align}
  &\text{PowerNorm}(\sigma_i; \beta, \Delta\boldsymbol{\beta})\nonumber\\
  &\quad=(1-\beta-\Delta{\beta})\sigma_i^{0.5}+(\beta+\Delta{\beta})\sigma_i^{0.75},
\end{align}
where $0\leq\beta\leq 1$ controls the steepness of the balancing push-forward curve and $\Delta{\beta}$ is 
drawn from the Normal distribution (we choose the best augmentation variance). 
We ensure that $0\leq\beta+\Delta{\beta}\leq 1$. 
In practice, we achieve the above function by Newton-Schulz iterations. Specifically, we have:
%
\begin{align}
\!\!\!\!\widetilde{\mathbf{H}}&=\mathbf{H}\mathbf{B}\big( (1-\beta\!-\!\Delta{\beta})\mathbf{A}^{0.5} +(\beta\!+\!\Delta{\beta})\mathbf{A}^{0.5}\mathbf{A}^{0.25}\big)\nonumber\\
&=\mathbf{H}\mathbf{B}\,\text{PowerNorm}(\mathbf{A}; \beta\!+\!\Delta{\beta}),
\end{align}
%
where $\mathbf{A}^{0.5}$ and $\mathbf{A}^{0.25}=\left(\mathbf{A}^{0.5}\right)^{0.5}$  are obtained by Newton-Schulz iterations.

\section{Dataset Description}
\label{sec:dataset}
Below, we describe the datasets and their statistics in detail:
\begin{itemize}
    \item
    \textbf{Cora, CiteSeer, PubMed.} These are well-known citation network datasets, in which nodes represent publications and edges indicate their citations. All nodes are labeled according to paper subjects~\citelatex{kipf2016semi_sup}.
    \item 
    
    \textbf{WikiCS.} It is a network of  Wikipedia pages related to computer science, with edges showing cross-references. Each article is assigned to one of 10 subfields (classes), with characteristics computed using the content's averaged GloVe embeddings~\citelatex{mernyei2020wiki_sup}.
    \item 
    
    \textbf{Am-Computer, AM-Photo.} Both of these networks are based on Amazon's co-purchase data. Nodes represent products, while edges show how frequently they were purchased together. Each product is described using a Bag-of-Words representation based on the reviews (node features). There are ten node classes (product categories) and eight node classes (product categories), respectively~\citelatex{mcauley2015image_sup}.
    
\end{itemize}
\begin{table}[h]
\caption{The dataset statistics.}
\label{tab:DatasetStats}
\centering
\resizebox{0.48\textwidth}{!}{
\begin{tabular}{l|ccccc}
\toprule
\textbf{Dataset} & \textbf{Type} & \textbf{Edges} & \textbf{Nodes} & \textbf{Attributes} & \textbf{Classes} \\
\midrule
Amazon-Computers & co-purchase & 245,778 & 13,381 & 767   & 10 \\
Amazon-Photo     & co-purchase & 119,043 & 7,487  & 745   & 8 \\
WikiCS           & reference   & 216,123 & 11,701 & 3,703 & 10 \\
Cora             & Citation    & 4,732   & 3,327  & 3,327 & 6 \\
CiteSeer         & Citation    & 5,429   & 2,708  & 1,433 & 7 \\
\bottomrule
\end{tabular}
}
\end{table}

\subsection{Implementation Details\label{sec:impl}}
\vspace{0.1cm}
\noindent \textbf{Image Classification.}
We use ResNet-18 as the backbone for image classification. We run 1000 epochs for CIFAR dataset and 400 epochs for imagenet-100. We follow the hyperparemeter setting in ~\citelatex{JMLR:v23:21-1155_sup}.

\vspace{0.1cm}
\subsection{Hyper-parameter Setting for node classification and clustering}
\label{app:hyperparam}
Below, we describe the hyperparameters we use for the main result. They are shown in Table \ref{tab:ModelAugSetting}.
\begin{table*}[h]
\centering
\resizebox{0.9\textwidth}{!}{  
\begin{tabular}{l|lllllllll}
\toprule\textbf{Dataset}  &\textbf{Hidden dims}. &\textbf{Projection Dims}.& \textbf{Learning rate}& \textbf{Training epochs} &$\mathbf{k}$ & $\mathbf{\tau}$ & \textbf{Acivation func.}\\
\midrule 
Cora         & $256$   & $256$    &$0.0001$ &$1,000$  &$1$ &$0.4$ & ReLU\\
CiteSeer     & $256$   & $256$    &$0.001$  &$500$   &$2$ &$0.9$ & ReLU\\
Am-Computer  & $128$   & $128$    &$0.01$   &$2,000$  &$1$ &$0.2$ &ReELU\\
Am-Photo     & $256$   & $64$     &$0.1$    &$2,000$  &$1$ &$0.3$ &ReLU\\
WikiCS       & $256$   & $256$    &$0.01$   &$3,000$  &$1$
&$0.4$ &PReLU\\
\bottomrule
\end{tabular}
}
\caption{The hyperparameters used for reproducing our results.}
\label{tab:ModelAugSetting}
\end{table*}
\subsection{Baseline Setting}
Table \ref{tab:bsc} shows the augmentation and constrictive objective setting of the major baseline of GSSL. Graph augmentations include: Edge Removing (ER), 
Personalized PageRank (PPR), Feature Masking (FM) and Node shuffling(NS). The contrastive objectives include: Information Noise Contrastive Estimation (InfoNCE),  Jensen-Shannon Divergence (JSD),  the Bootstrapping Latent loss (BL) and Barlow Twins (BT) loss.
\label{app:modelsetting}
\begin{table}[h]
\resizebox{0.45\textwidth}{!}{
\begin{tabular}{l|ccccccc}
\toprule \textbf{Method}  & \textbf{Graph Aug.} & \textbf{Spectral Feature Aug.} & \textbf{Contrastive Loss} \\
\midrule 
DGI    & NS   & $/$   & JSD \\
MVGRL  & PPR   & $/$   & JSD \\
GRACE  & ER+FM & $/$   & InfoNCE \\
GCA    & ER+FM & $/$   & InfoNCE \\
BGRL   & ER+FM & $/$   & BL \\
G-BT   & ER+FM & $/$   & BT \\
\midrule
SFA$_{\text{InfoNCE}}$ & ER+MF & $\checkmark$  & InfoNCE \\
SFA$_{\text{BT}}$  & ER+MF & $\checkmark$  & BT \\
\bottomrule
\end{tabular}
}
\caption{The setting of augmentation and contrastive objective of the major baseline and our approach.}\label{tab:bsc}
\end{table}

\section{Additional Empirical Results}
\label{sec:moreresult}

\noindent\textbf{Node Clustering.}
We also evaluate the proposed method on node clustering on Cora, Citeseer, and Am-computer datasets. We compare our model with t variational GAE (VGAE)~\citelatex{kipf2016variational_sup}, GRACE~\citelatex{DBLP:journals/corr/grace_sup}, and G-BT~\citelatex{bielak2021graph_sup}. We measure the performance by the clustering Normalized Mutual Information (NMI) and Adjusted Rand Index (ARI). We run each method 10 times. 
Table~\ref{tab:MainResultNodeClustering}  shows that our model achieves the best NMI and ARI scores on all benchmarks. To compare with a recently published contrastive model (\ie, SUGLR~\citelatex{mo2022simple_sup} and MERIT~\citelatex{jin2021multi_sup})), we combine SFA with these models and report their accuracy in Table~\ref{tab:node_classification_acc}.

\begin{table*}[!htbp]
\centering
\resizebox{0.7\textwidth}{!}{
\begin{tabular}{lcccccc}
\toprule
 & \multicolumn{2}{c}{ \textbf{Am-Computer} } & \multicolumn{2}{c}{ \textbf{Cora} } & \multicolumn{2}{c}{ \textbf{CiteSeer}} \\
\midrule \textbf{Method} & \textbf{NMI\%} & \textbf{ARI\%} & \textbf{NMI\%} & \textbf{ARI\%} & \textbf{NMI\%} & \textbf{ARI\%}  \\
\midrule 
K-means & $19.2\pm1.9$ &$ 8.6\pm1.3$   &$22.50\pm1.2$  &$22.1\pm2.1$  &$18.04\pm1.4$  &$18.33\pm1.8$\\
GAE     & $44.1\pm0.0$ &$25.8\pm1.0$   &$30.82\pm1.2$  &$25.47\pm0.9$ &$24.05\pm1.4$  &$20.56\pm0.8$\\
VGAE    & $42.6\pm0.1$ &$24.6\pm0.1$   &$32.92\pm1.2$  &$26.47\pm1.1$ &$26.05\pm0.9$  &$23.56\pm0.8$\\
\midrule 
GRACE   & $42.6\pm2.1$ &$24.6\pm1.3$   &$57.00\pm1.2$  &$53.74\pm1.5$ &$46.12\pm1.3$  &$44.50\pm1.1$ \\
G-BT & $ 41.0 \pm  1.9$ & $ 23.4\pm1.4$ & $54.90\pm2.2$ &$55.20\pm1.6$ &$45.00\pm2.3$ &$44.10\pm2.1$ \\
\midrule 
\rowcolor{LightCyan} SFA$_{\text{BT}}$ & $\textcolor{blue}{45.4 \pm 2.1}$ & $\textcolor{blue}{\mathbf{28.3 \pm 2.0}}$  &$\textcolor{blue}{55.91\pm2.3}$ &$\textcolor{blue}{55.36\pm1.7}$ &$\textcolor{blue}{46.46\pm2.3}$ &$\textcolor{blue}{\mathbf{45.37\pm1.8}}$ \\
\rowcolor{LightCyan} SFA$_{\text{InfoNCE}}$ & $\textcolor{blue}{\mathbf{46.2\pm2.2}}$ & $\textcolor{blue}{27.4\pm1.7}$ &$\textcolor{blue}{\mathbf{58.91}\pm1.1}$ &$\textcolor{blue}{\mathbf{56.96\pm1.6}}$ & $\textcolor{blue}{\mathbf{46.96\pm1.5}}$ &$\textcolor{blue}{44.97\pm1.8}$ \\
\bottomrule
\end{tabular}
}
\caption{Node Clustering result on Am-Computer, Cora, CiteSeer. \color{DarkBlue}{(Note that SFA$_{\text{InfoNEC}}$ and SFA$_{\text{BT}}$ can be directly compared with GRACE and G-BT respectively.)}\label{tab:MainResultNodeClustering}}
\end{table*}
   
\begin{table}[t]
    \centering
    \resizebox{0.3\textwidth}{!}{
    \begin{tabular}{l|ccc}
\toprule
  Accuracy            &    AM-Comp.    &       Cora    & Citeseer \\
\midrule
   K-means              &    $37.3$      &      $34.6$    & $38.4$   \\    
   GAE                  &    $54.0$      &      $41.2$    & $41.2$   \\     
   VGAE                 &    $55.4$      &      $55.9$    & $44.3$     \\
   \midrule   
   MERIT                &    $64.2$      &      $72.6$    & $69.8$  \\
   \rowcolor{LightCyan}SFA$_{\text{MERIT}}$ &   $\bf\textcolor{blue}{66.2}$       &      $\bf\textcolor{blue}{74.6}$    & $\bf\textcolor{blue}{70.9}$  \\
   \midrule
   MVGRL                &    $63.1$      &      $73.4$     & $70.1$  \\
   \rowcolor{LightCyan}SFA$_{\text{MVGRL}}$ &   $\bf\textcolor{blue}{66.0}$       &       $\bf\textcolor{blue}{73.9}$    & $\bf\textcolor{blue}{70.7}$  \\
   \midrule
   G-BT                 &    $62.4$      &      $65.6$    & $67.1$    \\
   \rowcolor{LightCyan}SFA$_{\text{BT}}$    &    $\bf\textcolor{blue}{65.6}$      &      $\bf\textcolor{blue}{67.4}$    & $\bf\textcolor{blue}{68.2}$ \\ 
   \midrule
   GRACE                &    $64.4$      &      $66.6$    & $68.1$     \\
   \rowcolor{LightCyan}SFA$_{\text{InfoNCE}}$  & $ \bf\textcolor{blue}{66.3}$     &      $\bf\textcolor{blue}{69.5}$    & $\bf\textcolor{blue}{69.2}$\\
   \midrule
   SUGRL                &   $65.2$       &      $73.2$    &  $70.5$ \\ 
   \rowcolor{LightCyan}SFA$_{\text{SUGRL}}$ &   $\bf\textcolor{blue}{66.7}$    &      $\bf\textcolor{blue}{74.4}$ & $\bf\textcolor{blue}{70.7}$ \\ 
 \bottomrule
    \end{tabular}
    }
    \caption{Node Clustering (accuracy \%).}
    \label{tab:node_classification_acc}
\end{table}

\vspace{0.1cm}
\noindent\textbf{Improved Alignment.}
We show more results of how SFA improves the alignment of two views during training (see Fig.~\ref{fig:more_ima}).
\begin{figure}[h]
\centering
\begin{subfigure}[b]{0.23\textwidth}
\includegraphics[width=\textwidth]{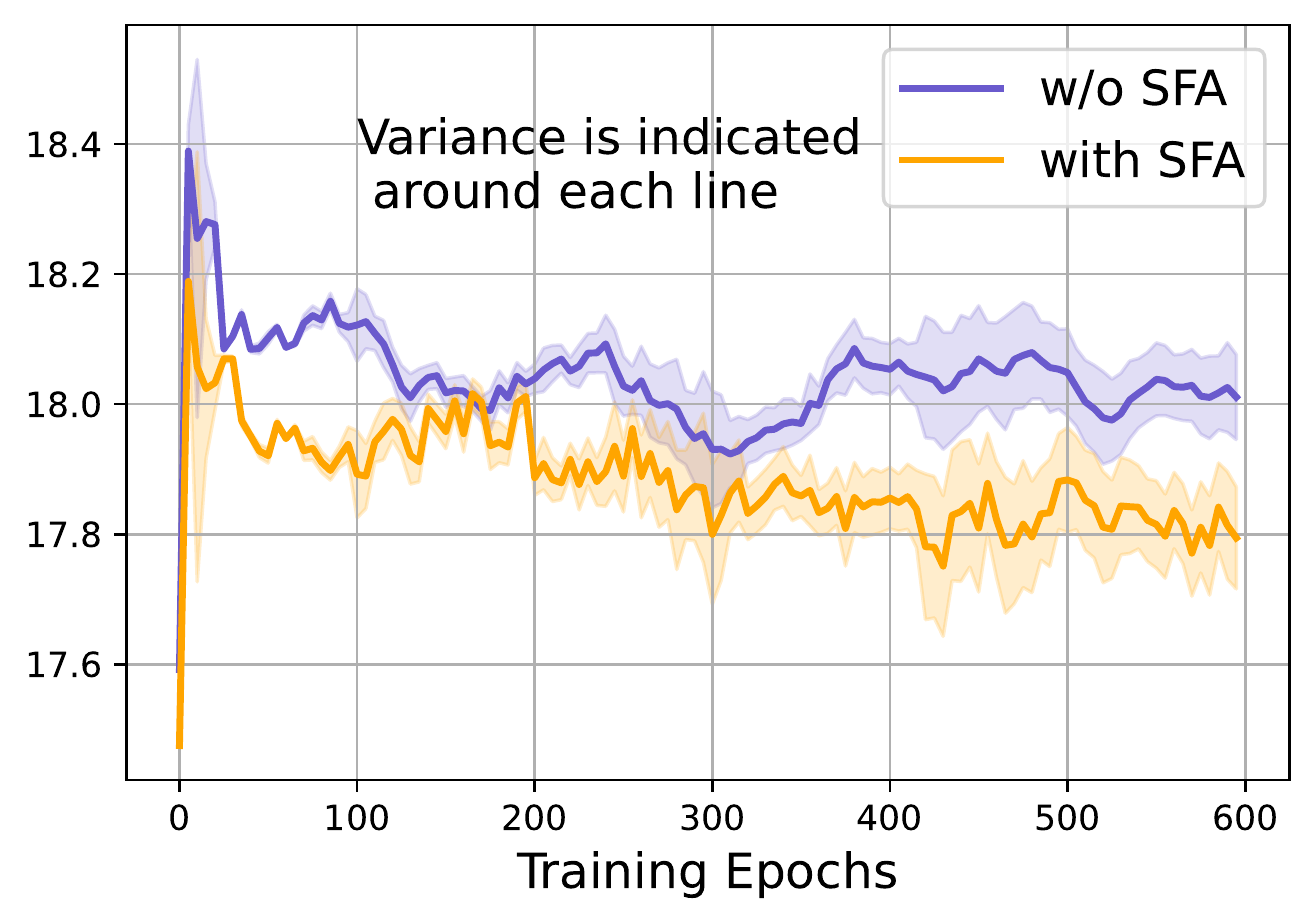}
\caption{CiteSeer}
\end{subfigure}
\hfill
\begin{subfigure}[b]{0.23\textwidth}
\includegraphics[width=\textwidth]{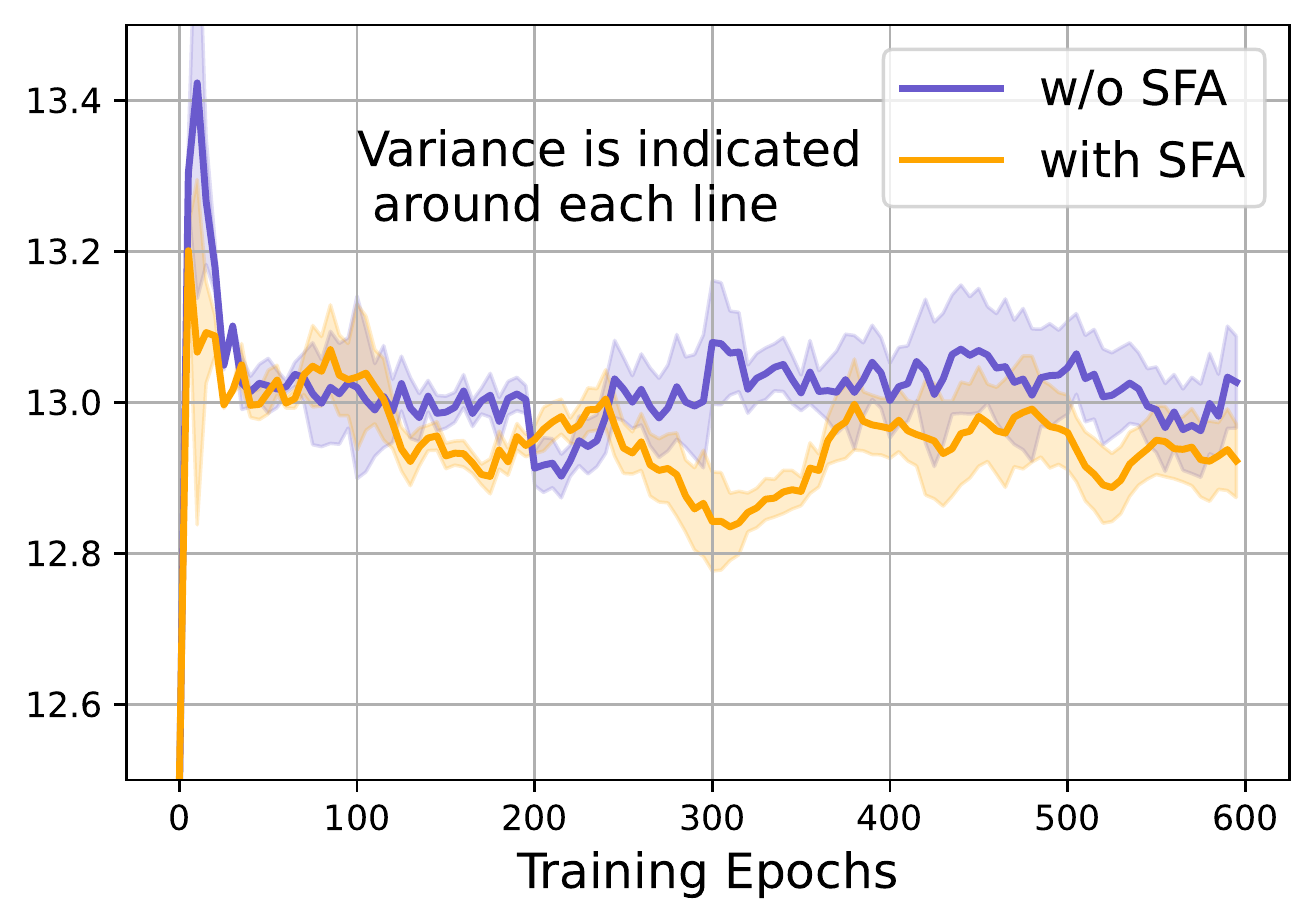}
\caption{WikiCS}
\end{subfigure}
\caption{Additional result on improved alignment on CiteSeer and WikiCS. Y-axis denote the alignment which is computed by $\|\mathbf{H}^\alpha -\mathbf{H}^\beta\|^2_F$ (or $\|\tilde{\mathbf{H}}^\alpha -\tilde{\mathbf{H}}^\beta\|^2_F$). The lower the curve the better.\label{fig:more_ima}}
\end{figure}

\section{Matrix Preconditioning}
\label{sec:mp}




\begin{table}[h]
    \centering
    \resizebox{0.48\textwidth}{!}{
    \begin{tabular}{l|ccc}
    \toprule
           Method                     &   Cora        &  CiteSeer      &      PubMed   \\
      \midrule
         Matrix Precond. (by LU factorization) & $78.1 \pm0.2$       &$67.7\pm0.3$     & $83.4\pm0.1$    \\
      Grassman (w/o noise) ($\kappa=15$)       & $81.1 \pm0.2$       &$72.3\pm0.3$        & $83.2\pm0.2$    \\
      PCA                          & $80.0 \pm0.1$       &$70.5\pm0.5$     &$81.5\pm0.3$     \\    
      \rowcolor{LightCyan} SFA$_\text{infoNCE}$          &$\bf 86.43\pm0.1$    &$\bf 76.1\pm0.3$   &$\bf 86.2\pm0.2$ \\
      \bottomrule
    \end{tabular}
    }
    \caption{\label{tab:mp}Comparison of SFA with Matrix Precond., Grassman feature map, and PCA.}
\end{table}

\vspace{0.1cm}
\noindent\textbf{Matrix Preconditioning (MP).} For the discussion on Matrix Preconditioning (MP) techniques, ``Matrix Preconditioning Techniques and Applications'' book~\citelatex{chen2005matrix_sup} 
shows how to design an effective preconditioner matrix $\mathbf{M}$ in order to obtain a numerical solution for solving a large linear system: 
\begin{equation}
    \mathbf{Ax}=\mathbf{b} \rightarrow \mathbf{MAx}=\mathbf{Mb}
\end{equation}
The goal of Matrix Preconditioning is to make $\tilde{\mathbf{A}} = \mathbf{MA}$ to have small conditional number $p= \sigma_{\text{max}} / \sigma_{\text{min}}$, (\eg,  $\mathbf{M} \approx \mathbf{A^{-1}}$ and $\tilde{\mathbf{A}} \approx \mathbf{I}$) so that solving $\mathbf{MAx}=\mathbf{Mb}$ is much easier, faster and more accurate than solving $\mathbf{Ax}=\mathbf{b}$.

\vspace{0.1cm}
\noindent\textbf{Spectrum rebalancing via MP.} Since the $\tilde{\mathbf{A}}$ has small condition number $p= \sigma_{\text{max}} / \sigma_{\text{min}}$, MP may somehow achieve similar effect on balancing the spectrum. We set ${\bf A=\mathbf{H}^\top \mathbf{H}}$ as ${\bf A}$ is required to be SPD (needs to be invertible) which additionally complicates recovery of the rectangular feature matrix $\tilde{\bf H}$ from $\mathbf{MA}$. MP requires inversions for which backpropagation through SVD is unstable (SVD backpropagation fails under so-called non-simple eigenvalues due to undetermined solution). \\

Below we  provide clear empirical results and compare our method to Matrix Preconditioning, Grassman feature map (without noise injection) and PCA (Table~\ref{tab:mp}). Specifically: 
\begin{itemize}
    \item We recover the augmented feature $\tilde{\mathbf{H}}$ from $\tilde{\mathbf{H}}^\top \tilde{\mathbf{H}} = \mathbf{MH}^\top\mathbf{H}$ where the preconditioner $\mathbf{M} =\mathbf{U}^{-1} \mathbf{L}^{-1}$ (where $\mathbf{L}$, $\mathbf{U}$ is the LU factorization of $\mathbf{H}^\top\mathbf{H}$). 
    \item Judging from the role preconditioner ${\bf M}$ fulfills (fast convergence of the linear system), ${\bf M}$ cannot solve our problem, \ie, it cannot  flatten some $\kappa>0$ leading singular values (where signal most likely resides) while leaving remaining singular values unchanged (our method achieves that as Fig.~\ref{fig:sim2} of the main paper shows).
    \item Grassmann feature map which forms subspace from $\kappa>0$ leading singular values is in fact closer ``in principle'' to SFA than MP methods. Thus, we apply SVD and flatten $\kappa$ leading singular values $\sigma_i$ while nullifying remaining singular values $\sigma_i$. 
    \item We also compare our approach to PCA given the best $\kappa=20$ components.
\end{itemize}

The impact of subspace size $\kappa$ on results with Grassmann feature maps (without noise injection) is shown in Table~\ref{impact_subspace_size}. Number $\kappa=15$ agrees with our observations that leading 12 to 15 singular values should be flattened (Fig~\ref{fig:Eigen1} of the main paper). 

\vspace{0.2cm}
Our SFA performs better as:
\vspace{0.2cm}

\begin{itemize}
    \item It does not completely nullify non-leading singular values (some of them are still useful/carry signal).
    \item It does perform spectral augmentation on the leading singular values.
    \item  SFA does not rely on SVD whose backpropagation step is unstable (\eg, undefined for non-simple singular including zero singular values).
\end{itemize}
    
\begin{table}[h]
      \centering
      \resizebox{0.47\textwidth}{!}{
      \begin{tabular}{c|ccccccccc|c}
      \toprule
  $\kappa$ &2  & 5 & 10 & 15 & 20 & 30 &50 &100 & 250 &SFA$_\text{infoNCE}$ \\
  \midrule
           &$78.7$ &$80.1$ &$80.5$ &$81.1$ &$80.3$ &$80.5$ &$80.4$ &$80.2$ &$79.8$ & $\textcolor{blue}{\mathbf{86.4}}$\\
  \bottomrule
      \end{tabular}
      }
      \caption{\label{impact_subspace_size}Impact of subspace size $\kappa$ in Grassman feature map without the noise injection (Cora).}
\end{table}

It should be noted that most of random MP methods in \citelatex{halko2011finding_sup} need to compute the eigenvalues/eigenvectors which is time-consuming even with random SVD, and this performs badly in practice.  Suppose $k$ is the iteration number of a random approach, it then requires $O(kq)$ to get top $q$ eigenvectors (\eg, the random SVD in~\citelatex{halko2011finding_sup}). Then one must manually set eigenvectors to yield equal contribution (\ie, $\sigma_i=1$ for all $i$) which is what we do exactly for the Grassmann feature map in Table~\ref{impact_subspace_size}.

Other plausible feature balancing models which are outside of the scope of our work include the so-called democratic aggregation \citelatex{democratic_sup,Lin_2018_ECCV_sup}.

\section{Why does the Incomplete Iteration of SFA Work?~\label{sec:why_it_works}}

Let $\mathbf{M}=\mathbf{H}^\top\mathbf{H}$ and let the power iteration be expressed as $\mathbf{M}^k\mathbf{r}^{(0)}$ where $\mathbf{r}^{(0)}\sim\mathcal{N}(0;\mathbf{I})$, as in Algorithm \ref{algo:1}. Let $\mathbf{M}=\mathbf{V}\mathbf{\Sigma}^2\mathbf{V}^\top$ be an SVD of $\mathbf{M}$, then the eigenvalue decomposition admits $\mathbf{M}^{k}\mathbf{V}=\mathbf{V}\mathbf{\Sigma}^2$. Let $\mathbf{r}^{(0)}=\mathbf{V}\widetilde{\mathbf{r}}^{(0)}$ then $\mathbf{M}^{k}\mathbf{r}^{(0)}=\mathbf{M}^{k}\mathbf{V}\widetilde{\mathbf{r}}^{(0)}=\mathbf{V}\mathbf{\Sigma}^{2k}\widetilde{\mathbf{r}}^{(0)}$. We notice that
\begin{equation}
\begin{aligned}
    \mathbf{M}^{k}\mathbf{r}^{(0)}&=\sigma_1^{2k}\sum_i \mathbf{v}_i\Big(\frac{\sigma_i}{\sigma_1}\Big)^{2k}\widetilde{{r}}_i^{(0)}\\
  &=\sigma_1^{2k}\sum_i\mathbf{v}_i\Big(\frac{\sigma_i}{\sigma_1}\Big)^{2k} \langle\mathbf{v}_i^\top, \mathbf{r}^{(0)}\rangle\\
  &=\sum_i\gamma_i\mathbf{v}_i,
    \end{aligned}
\end{equation}
where $\gamma_i=\sigma_1^{2k}\big(\frac{\sigma_i}{\sigma_1}\big)^{2k} \langle\mathbf{v}_i^\top,\mathbf{r}^{(0)}\rangle$ which clearly shows that $\sum_i\gamma_i\mathbf{v}_i$ is a linear combination of more than one $\mathbf{v}_i$ if $k\ll\infty$. Notice also that as the smaller $\sigma_i$ is compared to the leading singular value $\sigma_1$, the quicker the ratio  $\Big(\frac{\sigma_i}{\sigma_1}\Big)^{2k}$ declines towards zero as $k$ grows (it follows the  power law non-linearity in Figure \ref{fig:mr}). Therefore, leading singular values make a significantly stronger contribution to the linear combination $\sum_i\gamma_i\mathbf{v}_i$ compared to non-leading singular values.
Thus, for small $k$, matrix
\begin{equation}
\mathbf{M}_{\text{PowerIter}}=\frac{\mathbf{M}^{k}\mathbf{r}^{(0)}{\mathbf{r}^{(0)}}^\top\mathbf{M}^{k}}{\lVert \mathbf{M}^{k}\mathbf{r}^{(0)}\rVert_2^2}=\frac{\big(\sum_i\rho_i\mathbf{v}_i\big)\big(\sum_j\rho_j\mathbf{v}_j^\top\big)}{\lVert\sum_l\rho_l\mathbf{v}_l\rVert_2^2},
\end{equation}
is a low-rank matrix ``focusing'' on leading singular vectors of $\mathbf{M}$ more according to the power law non-linearity in Figure \ref{fig:mr}. Intuitively, for that very reason, as a typical matrix spectrum of $\mathbf{H}$ (and $\mathbf{M}$) also follows the power law non-linearity, $\mathbf{H}$ is balanced by the inverted power law non-linearity of $\mathbf{I}-\mathbf{M}_{\text{PowerIter}}$.
In the above equation, $\rho_i=\big(\frac{\sigma_i}{\sigma_1}\big)^{2k} \langle\mathbf{v}_i^\top,\mathbf{r}^{(0)}\rangle$.

\begin{figure}
    \centering
    \includegraphics[width=0.35\textwidth]{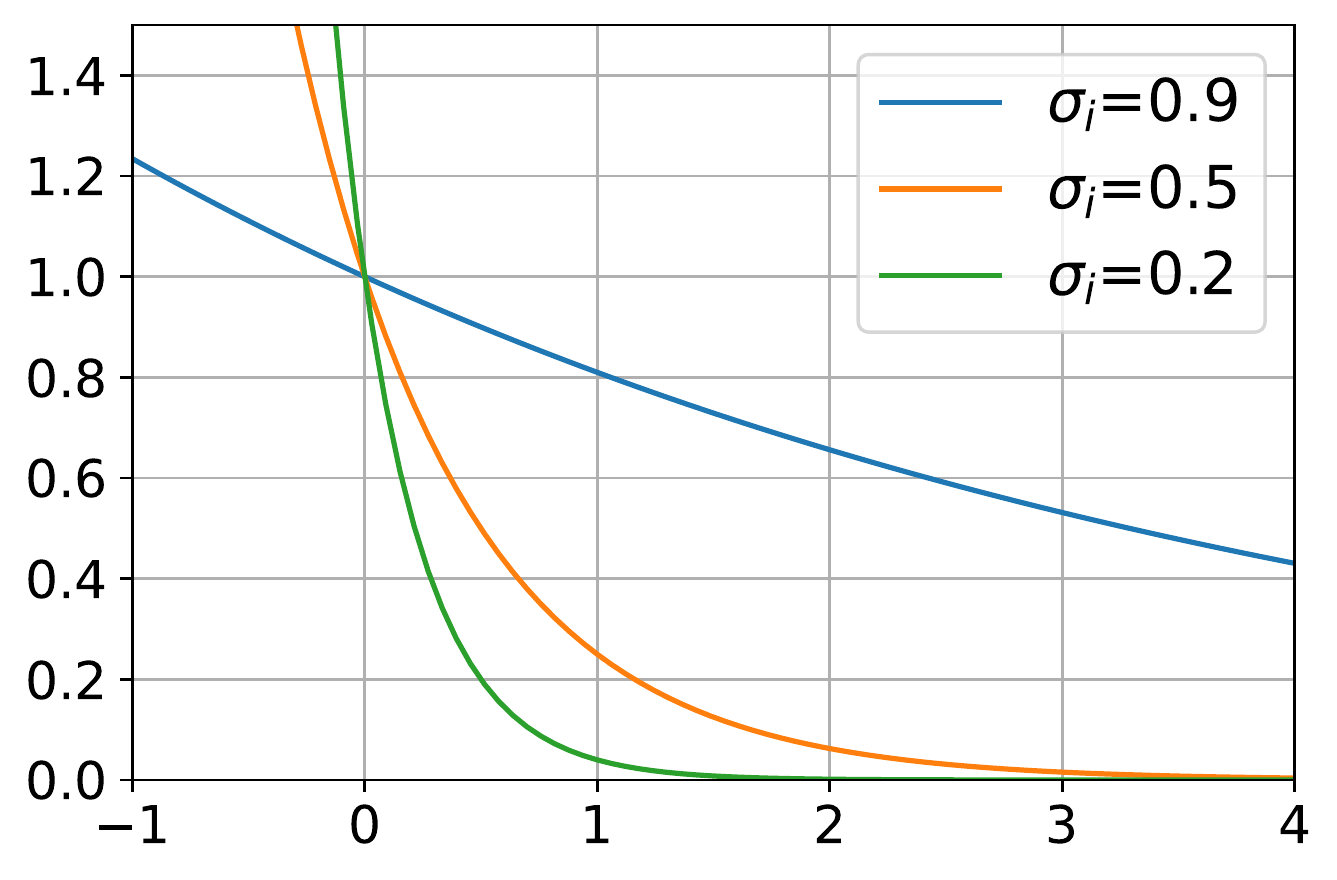}
    \caption{Plot of $\big(\frac{\sigma_i}{\sigma_1}\big)^{2k}$ given $\sigma_1=1$ exhibits the power law non-linearity.\label{fig:mr}}
\end{figure}

\section{Broader Impact and Limitations}
Empowering deep learning with the ability of reasoning and making predictions on the graph-structured data is of broad interest, GCL may be applied in many applications such as recommendation systems, neural architecture search, and drug discovery. The proposed graph contrastive learning framework with spectral feature augmentations is a general framework that can improve the effectiveness and efficiency of graph neural networks through model pre-training. It can also inspire further studies on the augmentation design from a new perspective. Notably, our proposed SFA is a plug-and-play layer which can be easily integrated with existing systems (\eg, recommendation system) at a negligible runtime cost. SFA facilitates training of high-quality embeddings for users and items to resolve the cold-start problem in on-line shopping (and other graph-based applications). One limitation of our method is that our work mainly serves as a plug-in for existing machine learning models and thus it does not model any specific fairness prior. Thus, the GCL model with SFA is still required to prevent the bias of the model (\eg, gender bias, ethnicity bias, \etc), as the provided data itself may be strongly biased during the processes of the data collection, graph construction, \etc. Exploring the augmentation with some fairness prior may address such a limitation. We assume is reasonable to let the GCL model tackle the fairness issue. 

\bibliographystylelatex{plainnat}
\bibliographylatex{aaai23sub}
\end{document}